\documentclass[11pt,oneside,reqno]{amsart}

\usepackage[utf8]{inputenc}
\usepackage[T1]{fontenc}
\usepackage{adjustbox, algorithm, algpseudocode, amsfonts, amssymb, amsthm, amsmath, array, authblk, bm, booktabs, braket, caption, changepage, comment, dsfont, enumitem, etoolbox, fullpage, graphicx, mathrsfs, mathtools, microtype, multicol, multirow, nicefrac, pifont, pythonhighlight, relsize, subfloat, tabularx, textcomp, url, xcolor, xpatch}

\usepackage{bbm}
\usepackage{float}
\usepackage{placeins}
\usepackage{makecell}
\usepackage{tikz, tikz-cd}
\usetikzlibrary{arrows.meta, positioning, shapes.geometric, decorations.pathreplacing}

\usepackage{subcaption}

\usepackage[round]{natbib}
\usepackage{bibunits}

\setlist[enumerate]{leftmargin=1.5em, itemsep=0.5em}
\setlist[itemize]{leftmargin=1.5em, itemsep=0.5em}

\usepackage[foot]{amsaddr}

\usepackage{titlesec}

\titleformat{\section}
	{\titlerule\vspace{-2ex}}
	{\thesection.}{1ex}{\centering\scshape}
	[\vspace{.5ex}\titlerule]

\titleformat{\subsection}[runin]
  {\normalfont\normalsize\bfseries}{\thesubsection.}{1ex}{}[.]
\titlespacing*{\subsection}{0pt}{0.0\baselineskip}{0.5ex}

\titleformat{\subsubsection}[runin]
  {\normalfont\normalsize\itshape}{\thesubsubsection.}{1ex}{}
\titlespacing*{\subsubsection}{0pt}{0.0\baselineskip}{0.5ex}

\newtheorem{theorem}{Theorem}[section]

\newtheoremstyle{style}
  {\baselineskip}
  {0em}
  {\itshape}
  {}
  {\bfseries}
  {.}
  {.5em}
  {}
\theoremstyle{style}

\newtheorem*{theorem*}{Theorem}

\newtheorem*{definition*}{Definition}

\newtheorem{lemma}{Lemma}[section]
\newtheorem*{lemma*}{Lemma}
\newtheorem{prop}[theorem]{Proposition}
\newtheorem*{prop*}{Proposition}

\usepackage{hyperref}
\hypersetup{
	colorlinks=true,
	linkcolor=black,
	citecolor=black,
	urlcolor=blue}

\usepackage[noabbrev,capitalise]{cleveref}
\creflabelformat{equation}{#2\textup{#1}#3}
\crefname{assumption}{Assumption}{Assumptions}
\Crefname{assumption}{Assumption}{Assumptions}
\crefname{cor}{Corollary}{Corollaries}
\Crefname{cor}{Corollary}{Corollaries}
\crefname{corollary}{Corollary}{Corollaries}
\Crefname{corollary}{Corollary}{Corollaries}
\crefname{definition}{Definition}{Definitions}
\Crefname{definition}{Definition}{Definitions}
\crefname{example}{Example}{Examples}
\Crefname{example}{Example}{Examples}
\crefname{prop}{Proposition}{Propositions}
\Crefname{prop}{Proposition}{Propositions}
\crefname{remark}{Remark}{Remarks}
\Crefname{remark}{Remark}{Remarks}

\algrenewcommand\algorithmicrequire{\textbf{Input:}}
\algrenewcommand\algorithmicensure{\textbf{Output:}}

\usepackage[colorinlistoftodos, textsize=tiny, textwidth=2.2cm, backgroundcolor=blue!10, linecolor=magenta, bordercolor=magenta]{todonotes}

\allowdisplaybreaks
\title{{\Large J}{\large ust add noise}\\\vspace{0.5em} {\normalsize D}{\small ebiasing tree-based variable importance in mixed data}}

\author{{\large J}iahe {\large L}i\textsuperscript{1}}
\author{{\large O}mar {\large M}elikechi\textsuperscript{1,*}}

\address{\textsuperscript{1}Department of Statistical Science, Duke University}
\address{\textsuperscript{*}Corresponding author: omar.melikechi@duke.edu}

\begin{document}

\frenchspacing

\begin{bibunit}[ims]

\date{}
\maketitle

\begin{abstract}

Variable importance scores from tree-based methods such as random forests favor continuous predictors over categorical ones. We present a theoretical analysis of this bias and propose a simple remedy: add a small amount of noise to each categorical predictor. The correction is demonstrated on a variety of simulated and real-world datasets and combined with integrated path stability selection to perform variable selection with false discovery control for mixed data.

\end{abstract}


\section{Introduction}\label{sec:intro}


Variable importance scores from tree-based algorithms such as random forests \citep{breiman2001random} and gradient boosting \citep{friedman2001greedy} are routinely used to identify relevant features in supervised prediction problems. When datasets contain both continuous and categorical predictors, however, standard importance scores from such methods are biased, often favoring continuous features over categorical ones, and categorical features with more levels over those with fewer \citep{strobl2007bias}. Selecting features based on these scores can therefore be misleading in mixed data settings, which are ubiquitous in medicine, engineering, and the physical and social sciences \citep{solorio2022survey,warkiani2025comprehensive}.

Despite extensive empirical evidence of this bias, there remains little theory about mixed data in tree-based algorithms. Most theoretical results about tree-based importance scores assume that all predictors are continuous \citep{li2019debiased,scornet2023trees,ramosaj2023consistent}. One exception is \citet{louppe2013understanding}, who establish asymptotic properties of mean decrease impurity (MDI) when all features and the response are categorical. However, in addition to excluding continuous features, their results are restricted to totally randomized trees and thus do not apply to the classification and regression trees (CART) upon which many tree-based methods are built \citep{breiman1984classification}.

In this paper, we formally analyze the longstanding but unproven claim that mixed data MDI bias results from a fundamental asymmetry in CART and similar decision tree algorithms: at a node containing $n$ samples, continuous variables admit $n-1$ possible splits almost surely, while the number of ways to split a categorical variable depends only on the number of categories it has (\cref{fig:structural_asymmetry}). We first establish results when all predictors are independent of the response---the \textit{null case}---showing that CART assigns all importance to continuous predictors as $n\to\infty$. We then characterize MDI bias in the model $Y = f(X) + g(Z) + \varepsilon$, where $X$ is continuous, $Z$ is binary, and $\varepsilon$ is noise. Our analyses show that---because $Z$ can only be split once---the residual variance from $\varepsilon$ is necessarily absorbed by $X$, spuriously inflating its importance (\cref{fig:nonnull_allocation}).

After presenting our theoretical results, we propose a simple solution to the mixed data problem: add a small amount of noise to each categorical feature. Simulation and real data results show that this approach is remarkably effective at removing MDI bias. Furthermore, we find that importance scores are insensitive to the magnitude of the noise and that adding noise generally preserves, or even improves, predictive performance relative to unperturbed baselines. Finally, adding noise is trivial to implement, computationally inexpensive, and seamlessly integrates with virtually any software.

\begin{figure}[htbp]
    \centering
    \includegraphics[width=0.85\linewidth]{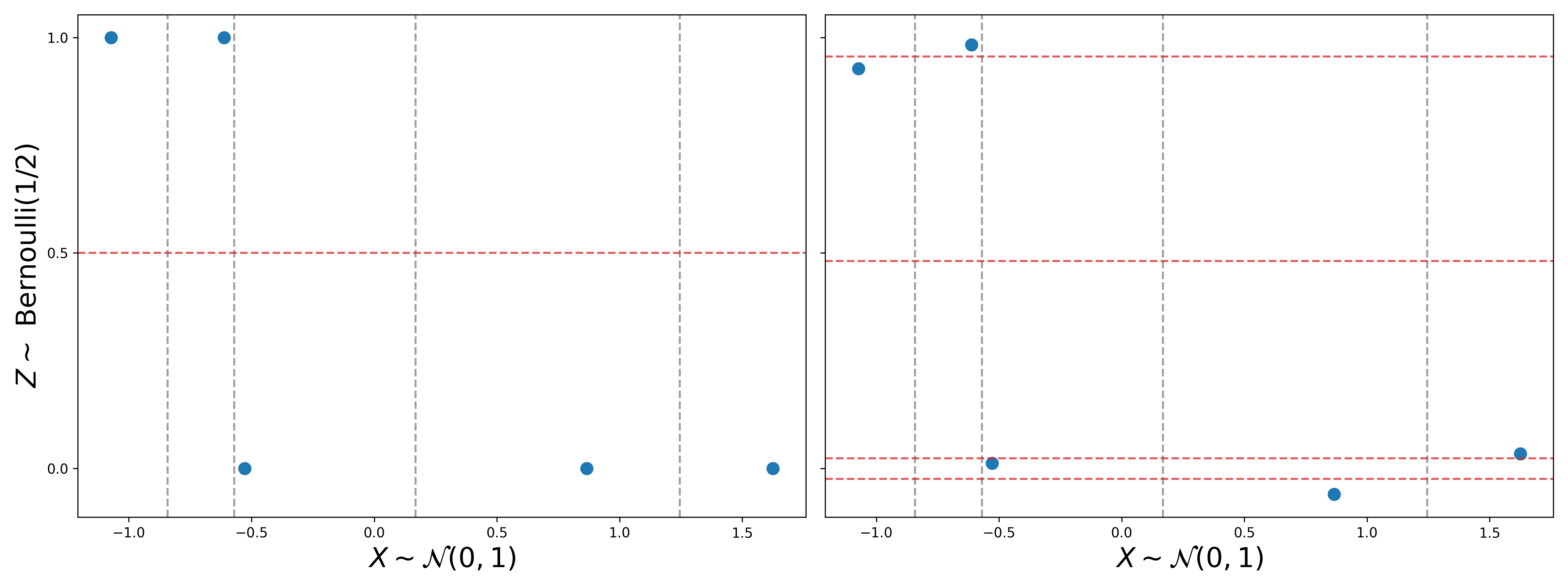}
    \caption{\textit{Adding noise yields equal numbers of candidate splits}. Five samples are drawn from the independent pair $(X,Z)$, where $X\sim\mathcal{N}(0,1)$ and $Z\sim\text{Bernoulli}(1/2)$. (\textit{Left}) The 5 samples of $X$ admit 4 splits (vertical lines), while the $Z$ samples only admit 1 (horizontal line). (\textit{Right}) The $Z$ samples admit 4 splits after adding noise.}
    \label{fig:structural_asymmetry}
\end{figure}

\begin{figure}[htbp]
    \centering    \includegraphics[width=0.75\linewidth]{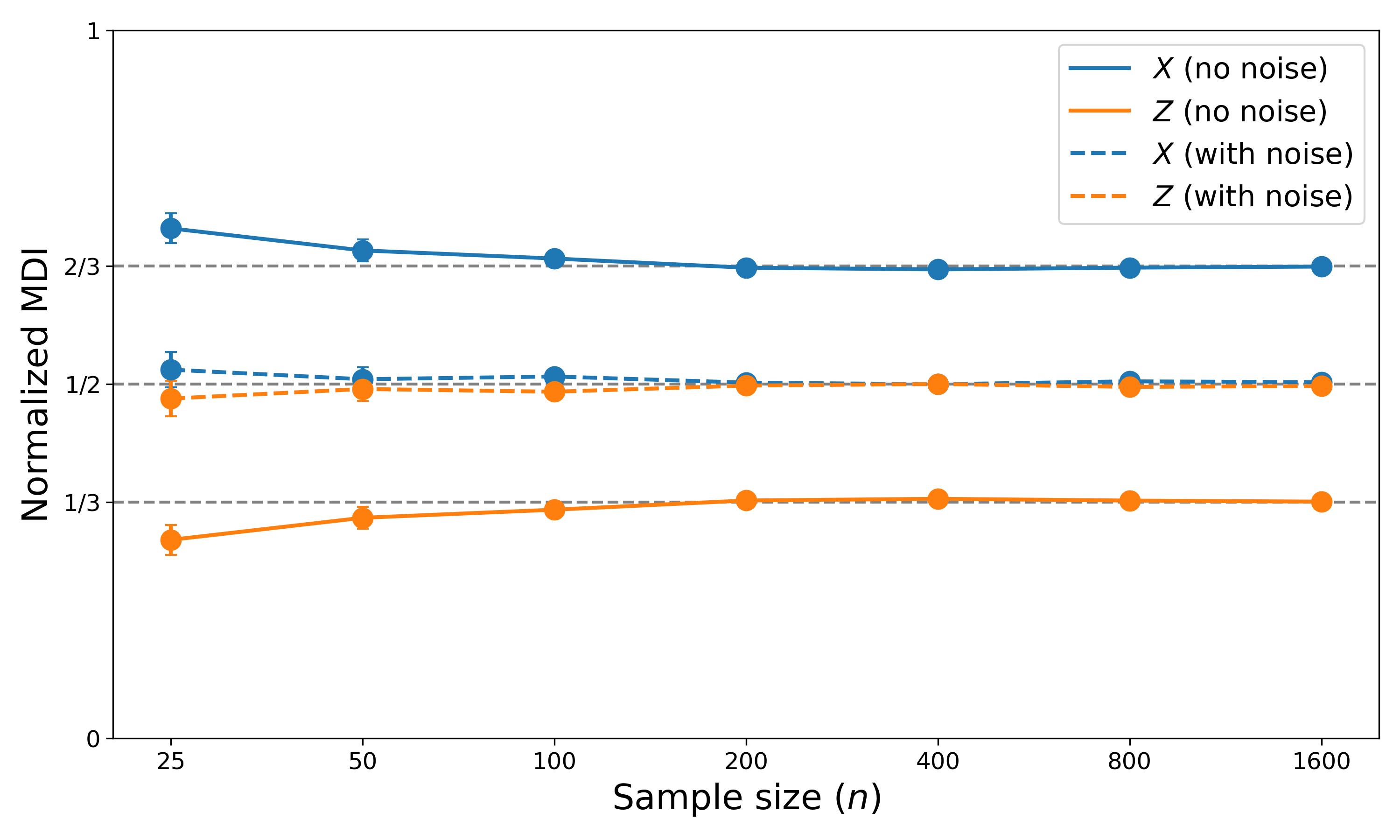}
    \caption{\textit{Mixed data bias.} MDI scores from CART, plotted as a function of $n$ for the model $Y=X+Z+\varepsilon$, where $X\sim N(0,1)$, $Z=\pm1$ with equal probability, and $\varepsilon\sim N(0,1)$ are mutually independent. Although $\operatorname{Var}(X)=\operatorname{Var}(Z)=1$, normalized MDIs for $X$ and $Z$ converge to $2/3$ and $1/3$, respectively, in accordance with \cref{thrm:nonnull_allocation}. After adding noise to $Z$, both scores converge to $1/2$. Scores are averaged over 200 trials; error bars show 95\% Monte Carlo confidence intervals for the mean.}
    \label{fig:nonnull_allocation}
\end{figure}


\subsection*{Related work}


Several methods have been proposed to correct MDI bias in mixed data. Conditional inference forests \citep[CForest;][]{hothorn2006unbiased,strobl2007bias} replace impurity-based greedy split selection with hypothesis tests for association at each node; this removes the selection bias but requires permutation tests at every split, making the method computationally prohibitive even in moderate dimensions (\cref{tab:runtime_n}). Out-of-bag corrections \citep{li2019debiased,zhou2021unbiased} recompute impurity decreases on held-out observations and are substantially faster. However, unlike adding noise, all of these approaches require modifying either the tree-building algorithm or the importance scoring step, making them less directly compatible with the standard implementations of tree-based algorithms that are deeply embedded in existing statistical and machine learning software.


\subsection*{Organization}


In \cref{sec:theory}, we present theoretical results about MDI bias in mixed data. In \cref{sec:just_add_noise}, we introduce and discuss adding noise. \cref{sec:simulations} presents simulation studies and \cref{sec:realdata} applies the method to real data. \cref{sec:discussion} concludes.


\section{Theory}\label{sec:theory}



\subsection{Setup and notation}\label{sec:setup}


For a positive integer $n$, define $[n] = \{1, \ldots, n\}$, and for a finite set $A$, let $|A|$ denote its size. We call $B$ a \textit{proper subset} of $A$ if $B$ is nonempty, $B \subseteq A$, and $B \neq A$; the nonemptiness requirement is not standard but will be convenient in what follows. Write $\mathcal{P}^\ast(A)$ for the collection of proper subsets of $A$ in this sense.

Assume we have $n$ independent and identically distributed (i.i.d.) samples $\{(x_i, z_i, y_i)\}_{i=1}^n$, defined on a common probability space $(\Omega, \mathcal{A}, \mathbb{P})$ with expectation $\mathbb{E}$. Throughout, $x_i = (x_{i1}, \ldots, x_{ip})$ is an observation of the continuous predictors $X = (X_1, \ldots, X_p) \in \mathbb{R}^p$; $z_i = (z_{i1}, \ldots, z_{iq})$ is an observation of the categorical predictors $Z = (Z_1, \ldots, Z_q) \in \mathcal{Z} = \mathcal{Z}_1 \times \cdots \times \mathcal{Z}_q$ with each $\mathcal{Z}_k$ a finite set of size $|\mathcal{Z}_k| \geq 2$; and $y_i$ is an observation of the response $Y\in\mathbb{R}$. We assume each coordinate $X_j$ is absolutely continuous with respect to Lebesgue measure on $\mathbb{R}$, so that $x_{1j}, \ldots, x_{nj}$ are almost surely distinct. We also assume $\operatorname{Var}(Y)\in (0,\infty)$. Throughout, $i \in [n]$ indexes samples, $j \in [p]$ continuous predictors, and $k \in [q]$ categorical predictors. We write $\mathbf{1}(\cdot)$ for the indicator function and $\overset{p}{\longrightarrow}$ for convergence in probability. For any $t \subseteq [n]$, set $S_t = \sum_{i \in t} y_i$, $n_t = |t|$, and $\bar y_t = S_t / n_t$.

Classification and regression trees (CART) are constructed by iteratively splitting nodes $t \subseteq [n]$ into proper subsets $t_L$ and $t_R = t \setminus t_L$, called \textit{children} of $t$, starting from the root node $t = [n]$. Given an impurity function $\mathcal{I}$, the decrease in impurity associated with $t_L$ (equivalently, $t_R$) is
\begin{equation}
\Delta \mathcal{I}(t, t_L)
= \mathcal{I}(t)
- \frac{|t_L|}{n_t} \mathcal{I}(t_L)
- \frac{|t_R|}{n_t} \mathcal{I}(t_R).
\label{eq:impurity_decrease}
\end{equation}
In this work, we take $\mathcal{I}$ to be the variance,
\begin{equation}
\mathcal{I}(t)
= \frac{1}{n_t} \sum_{i \in t}(y_i - \bar{y}_t)^2
= \frac{1}{n_t}\sum_{i \in t} y_i^2 - \frac{S_t^2}{n_t^2}.
\label{eq:variance_impurity}
\end{equation}
When $Y$ is binary, $\bar{y}_t$ is the proportion of samples in $t$ with $y_i=1$. Since $y_i^2 = y_i$, \cref{eq:variance_impurity} becomes $\mathcal{I}(t) = \bar{y}_t(1 - \bar{y}_t)$, which is one half the Gini index \citep{breiman1984classification}.

CART splits $t$ by identifying proper subsets $t_L$ that maximize $\Delta \mathcal{I}(t, t_L)$ with respect to each predictor. For a continuous predictor $X_j$, let $x_{(1)j}, \ldots, x_{(n_t)j}$ be the reordering of the observed data $\{x_{ij} : i \in t\}$ in increasing order and, for $1 \leq \ell < n_t$, define
\begin{align*}
L_\ell^{t}(X_j)
=
\Bigl\{ i \in t :
x_{ij} \leq \tfrac{1}{2}\bigl(x_{(\ell)j} + x_{(\ell+1)j}\bigr) \Bigr\}.
\end{align*}
In \cref{fig:structural_asymmetry}, for example, the midpoints $(x_{(\ell)j} + x_{(\ell+1)j})/2$ are represented by the vertical dashed lines. The maximum decrease in impurity along $X_j$ at node $t$ is then
\begin{equation}
M_{t}(X_j)
= \max_{1 \leq \ell < n_t}
\Delta\mathcal{I}\bigl(t, L_\ell^{t}(X_j)\bigr).
\label{eq:continuous_gain}
\end{equation}
For a categorical predictor $Z_k$ and any $B \in \mathcal{P}^\ast(\mathcal{Z}_k)$, define $L_B^{t}(Z_k) = \{i \in t : z_{ik} \in B\}$. Let $\mathcal{B}_t(Z_k)=\{B\in\mathcal{P}^\ast(\mathcal{Z}_k):0<|L_B^t(Z_k)|<n_t\}$ be the valid categorical splits at node $t$. The maximum decrease in impurity along $Z_k$ at $t$ is
\begin{equation}
M_{t}(Z_k)
= \max_{B \in \mathcal{B}_t(Z_k)}
\Delta\mathcal{I}\bigl(t, L_B^{t}(Z_k)\bigr).
\label{eq:categorical_gain}
\end{equation}
The maximum over an empty set is defined to be zero.
A key point in what follows is that, because $L_B^t(Z_k)$ and $L_{\mathcal{Z}_k \setminus B}^t(Z_k)$ induce identical splits, $Z_k$ admits at most $2^{|\mathcal{Z}_k|-1}-1$ ways to split $t$. By contrast, continuous variables admit $n_t-1$ possible splits almost surely, a number that grows with node size. The CART split at $t$ occurs along the predictor with the largest impurity decrease; thus a split occurs along a continuous variable if
\begin{align*}
\max_{1\leq j \leq p} M_{t}(X_j) > \max_{1 \leq k \leq q} M_{t}(Z_k).
\end{align*}

Finally, the \textit{mean decrease impurity} (MDI) of $X_j$ in an ensemble of decision trees $\mathcal{T}$ is
\begin{equation}
\Phi_n(X_j)
= \frac{1}{|\mathcal{T}|}
\sum_{T \in \mathcal{T}} \sum_{t \in T}
\frac{n_t}{n}\,M_{t}(X_j)\,
\mathbf{1}\bigl(t \text{ splits along } X_j\bigr),
\label{eq:mdi}
\end{equation}
where $M_{t}(X_j) = 0$ at terminal nodes. $\Phi_n(Z_k)$ is defined analogously with $Z_k$ in place of $X_j$.


\subsection{The null model}\label{sec:null}


We first quantify MDI bias when all predictors carry no signal. The \textit{null model} refers to the case where $Y$ is mutually independent of each predictor: $Y \perp X_j$ and $Y \perp Z_k$ for all $j \in [p]$ and $k \in [q]$. In the stronger \textit{joint null model}, $Y$ is jointly independent of all predictors: $Y \perp (X_1, \ldots, X_p, Z_1, \ldots, Z_q)$. Note that both models make no assumptions about the dependence structure between predictors. Our first result is about splits at the root node in the null model.

\begin{theorem}[Asymptotic dominance of continuous splits]\label{thrm:null_root}
Under the null model,
\begin{align*}
\lim_{n\to\infty}
\mathbb{P}\!\left(
\max_{1\leq j\leq p} M_n(X_j)
>
\max_{1\leq k\leq q} M_n(Z_k)
\right)
= 1,
\end{align*}
where $M_n(X_j) = M_{[n]}(X_j)$ and $M_n(Z_k) = M_{[n]}(Z_k)$ are the maximum impurity decreases at the root node, given by \cref{eq:continuous_gain,eq:categorical_gain} with $t = [n]$.
\end{theorem}

\cref{thrm:null_root} characterizes mixed data bias at the root node, stating that CART splits along a continuous variable with probability one as $n\to\infty$ under the null model. Under the joint null model, this bias extends to trees of bounded depth, provided every internal node contains a fixed positive fraction of the total number of samples, $n$.

\begin{theorem}[Vanishing importance of categorical predictors]
\label{thrm:tree_null}
Assume the joint null model and let $\mathcal{T}$ be a finite CART ensemble in which each tree is of bounded depth, $D\geq1$. Assume that, for some fixed $c\in(0,1)$, every internal node $t$ satisfies $n_t \geq cn$. Then the probability of splitting along a categorical variable at any node in any tree in the ensemble vanishes asymptotically; that is,
\begin{align*}
\lim_{n\to\infty}
\mathbb{P}\!\left(
\sum_{T\in\mathcal{T}}\sum_{t\in T}\sum_{k=1}^q
\mathbf{1}\bigl(t \text{ splits along } Z_k\bigr)>0
\right)=0.
\end{align*}
Consequently, for the MDI importances defined in \cref{eq:mdi},
\begin{align*}
\frac{\sum_{k=1}^q \Phi_n(Z_k)}
{\sum_{j=1}^p \Phi_n(X_j)+\sum_{k=1}^q \Phi_n(Z_k)}
&\overset{p}{\longrightarrow} 0
\qquad\text{and}\qquad
\frac{\sum_{j=1}^p \Phi_n(X_j)}
{\sum_{j=1}^p \Phi_n(X_j)+\sum_{k=1}^q \Phi_n(Z_k)}
\overset{p}{\longrightarrow} 1.
\end{align*}
\end{theorem}

The node-size condition, $n_t\geq cn$, is a minimum split-size requirement that appears as a parameter in many software implementations of tree-based algorithms, such as the \texttt{min\_samples\_split} parameter in \texttt{scikit-learn} \citep{pedregosa2011scikit}. Under this condition, the continuous predictors collectively receive all normalized MDI asymptotically, while categorical predictors receive none.

\cref{thrm:null_root,thrm:tree_null} establish a preference for continuous predictors over categorical ones, but leave open whether MDI discriminates among predictors of the same type. The following proposition shows it does not, at least in the case of exchangeable continuous predictors.

\begin{prop}[Equal expected normalized importance]\label[prop]{prop:fair}
Assume there are no categorical predictors and that $X_1,\ldots,X_p$ are exchangeable and jointly independent of $Y$. Let $\mathcal T$ be a finite CART ensemble with $\sum_{\ell=1}^p\Phi_n(X_\ell)>0$ almost surely. Then, for every $j\in[p]$,
\begin{align*}
\mathbb E\!\left[
\frac{\Phi_n(X_j)}
{\sum_{\ell=1}^p\Phi_n(X_\ell)}
\right]
=
\frac1p.
\end{align*}
\end{prop}

The condition $\sum_{\ell=1}^p\Phi_n(X_\ell)>0$ just states that at least one predictor has nonzero MDI. 


\subsection{A non-null model}\label{sec:nonnull}


We now move beyond null models and consider the additive model
\begin{align}\label{eq:non_null_model}
Y=f(X)+g(Z)+\varepsilon,
\end{align}
where $X=X_1$ is continuous, $Z=Z_1\in\{z_0,z_1\}$ is binary, and $X$, $Z$, and $\varepsilon$ are independent. We also assume $\mathbb E[f(X)^2]<\infty$, $\mathbb E[\varepsilon]=0$, $\operatorname{Var}(\varepsilon)=\sigma^2<\infty$, $\pi=\mathbb P(Z=z_0)\in(0,1)$ and $g(z_0)\ne g(z_1)$.

Let $M_*(X)$ be the population analogue of $M_n(X)$, that is, the largest population impurity decrease obtainable by splitting $X$. In \cref{lem:cont_gain}, we prove that $M_*(X)=\lim_{n\to\infty} M_n(X)$.

\begin{theorem}[Asymptotic MDI allocation]
\label{thrm:nonnull_allocation}
Assume the model in \cref{eq:non_null_model} and suppose that
\begin{align*}
\operatorname{Var}(g(Z))>M_*(X).
\end{align*}
Then the root node in CART splits on $Z$ with probability tending to one. Moreover, let $\Phi_n(X)$ and $\Phi_n(Z)$ denote the MDI importances defined in \cref{eq:mdi}, where $\mathcal T$ is any finite CART ensemble with each tree fully grown in the sense that every terminal node has zero impurity. Then
\begin{align*}
\Phi_n(Z)
\overset{p}{\longrightarrow}
\operatorname{Var}(g(Z))
\quad\text{and}\quad
\Phi_n(X)
\overset{p}{\longrightarrow}
\operatorname{Var}(f(X))+\sigma^2.
\end{align*}
Consequently, noting that $\operatorname{Var}(Y) = \operatorname{Var}(f(X))+\operatorname{Var}(g(Z))+\sigma^2$,
\begin{align*}
\frac{\Phi_n(Z)}
{\Phi_n(X)+\Phi_n(Z)}
\overset{p}{\longrightarrow}
\frac{\operatorname{Var}(g(Z))}
{\operatorname{Var}(Y)}
\quad\text{and}\quad
\frac{\Phi_n(X)}
{\Phi_n(X)+\Phi_n(Z)}
\overset{p}{\longrightarrow}
\frac{\operatorname{Var}(f(X)) + \sigma^2}
{\operatorname{Var}(Y)}.
\end{align*}
\end{theorem}

The condition $\operatorname{Var}(g(Z))>M_*(X)$ compares the population gain of the categorical split with the best possible population gain along $X$. Under this separation, the former eventually exceeds the latter, so the root splits along $Z$. Once $Z$ is split, it is constant in each child node and cannot be used again. Thus the MDI associated with $Z$ is $\operatorname{Var}(g(Z))$ asymptotically, while the remaining impurity reduction, $\operatorname{Var}(f(X)) + \sigma^2$, is allocated to $X$. The contribution of $\sigma^2$ to X's MDI importance is precisely the source of the bias illustrated in \cref{fig:nonnull_allocation}.


\section{Just add noise}\label{sec:just_add_noise}


In the previous section we showed that mixed data bias in MDI scores from CART and CART-based ensembles results from continuous predictors having more splits than categorical ones as $n$ grows. We also saw in \cref{prop:fair} that exchangeable continuous predictors have identical MDI under the joint null model, suggesting no such bias exists when all variables are continuous.

A simple way to make all predictors continuous---and thus to ensure all predictors have the same number of candidate splits almost surely---is to add continuous noise to each categorical variable. This procedure, which we call \textit{jittering}, is our proposed solution to the mixed data bias problem. For categorical predictors $Z_k$ with integer-encoded values, we add uniform noise, $\widetilde{z}_{ik} = z_{ik} + u_{ik}$, where each $u_{ik}$ is an independent draw from $\operatorname{Unif}(-\delta,\delta)$ and $\delta>0$ is the \textit{jitter strength}. Using the uniform distribution ensures distinct categorical levels remain well-separated whenever $\delta<0.5$. If $Z\in\{0,1\}$, for example, jittering will not cause observed zero and one values to overlap.


\subsection*{One-time and per-tree jittering}\label{sec:jittering_variants}


The definition above leaves open when the perturbations $u_{ik}$ are drawn. For tree ensembles, one can either jitter categorical variables before fitting the ensemble (\textit{one-time jittering}), or a fresh perturbation can be applied prior to building each individual tree (\textit{per-tree jittering}). One-time jittering is a simple preprocessing step: identify the categorical predictors, add $\operatorname{Unif}(-\delta,\delta)$ noise to each, and pass the resulting data to any prediction algorithm. Per-tree jittering instead draws noise inside the ensemble loop, making it no longer a preprocessing step. Results in \cref{sec:simulations} suggest that per-tree jittering produces slightly better rankings than one-time jittering, but the improvement is small relative to the convenience of a preprocessing step. We therefore recommend using one-time jittering in general.


\subsection*{Jitter strength}\label{sec:noise_level}


As noted above, any $\delta<0.5$ keeps integer-encoded categories separated, so that the original category is recoverable from the jittered value. Sensitivity analyses in \cref{supsec:sensitivity} show that results are highly robust to varying $\delta$: feature recovery and predictive performance are essentially unchanged as $\delta$ ranges over multiple orders of magnitude. In particular, $\delta$ does not require tuning. We use $\delta=10^{-4}$ throughout this work.


\subsection*{Effect on prediction}\label{sec:jitter_prediction}


Our aim is to reduce mixed data importance bias, but adding noise raises the question of whether it harms predictive accuracy. To test this, we compared random forest and XGBoost prediction errors to those of their jittered counterparts on 51 datasets from the Penn Machine Learning Benchmark \citep{Olson2017PMLB}. Results reported in \cref{supsec:prediction} indicate that jittering has a negligible effect on prediction: across all 51 datasets, jittering leaves the prediction error of both methods essentially unchanged (\cref{fig:prediction_error_pmlb}).


\subsection*{Computation}\label{sec:computation}


One-time and per-tree jittering are computationally inexpensive; for an ensemble of trees $\mathcal{T}$, they contribute $O(nq)$ and $O(nq\lvert\mathcal{T}\rvert)$ operations, respectively. Both approaches are faster than unbiased feature importance \citep[UFI;][]{zhou2021unbiased} and orders of magnitude faster than CForest, with both gaps widening as sample size grows (\cref{tab:runtime_n}).


\section{Simulation studies}\label{sec:simulations}


We conduct two types of simulation studies. In the first, we investigate the effect of jittering and other bias-reduction methods on raw MDI scores. In the second, we combine these importance scores with integrated path stability selection (IPSS) to perform variable selection with finite-sample false discovery control \citep{melikechi2026integrated}. Both studies use a common data-generating framework with $n=200$ samples, $p=25$ continuous predictors, and $q=25$ binary predictors. In each study, we consider independent and correlated covariate structures as well as linear and nonlinear responses, yielding four simulation settings per study. Each setting consists of $100$ trials.


\subsection{Simulation design}\label{sec:simulation_design}


In both the independent and correlated designs, the binary predictors $Z_1,\ldots,Z_q$ are i.i.d.\ $\operatorname{Bernoulli}(1/2)$. In the independent design the continuous predictors $X_1,\ldots,X_p$ are i.i.d.\ $N(0,1)$, independent of the binary predictors. In the correlated design the $j$th continuous predictor is paired with the $j$th binary predictor such that
\begin{align*}
    X_j\mid Z_j=1\sim N(1,1)
    \quad\text{and}\quad
    X_j\mid Z_j=0\sim N(-1,1),
\end{align*}
with $\operatorname{Corr}(X_j,Z_j)=1/\sqrt{2}$. In the correlated design a continuous predictor can serve as a proxy for its binary partner and compete with it for splits, making the binary signal harder to recover.

In each trial, five continuous and five binary predictors are sampled uniformly without replacement to form \textit{active sets} $S_X$ and $S_Z$. The response is then generated as
\begin{align*}
Y = \sum_{j\in S_X} r_j\,\psi(X_j) + 2\sum_{k\in S_Z} r_k\,Z_k + \varepsilon,
\end{align*}
where $\varepsilon\sim N(0,\sigma^2)$, the signs $r_j,r_k$ are drawn once per trial, independent and uniform on $\{-1,1\}$, and $\psi$ is either linear, $\psi(x)=x$, or nonlinear, namely $e^{-x^2/2}$ shifted and scaled to have mean zero and unit variance whenever $X\sim N(0,1)$ (\cref{eq:exponential}). Under the independent design every term has unit variance---the factor $2$ standardizes the binary effects---and $\sigma^2$ is set in each trial so that the signal-to-noise ratio is one. Additional details about the simulation design are in \cref{supsec:sim_dgp}.


\subsection{Ranking by importance scores}\label{sec:simulation_auc}


Our first simulation study measures how well importance scores rank active predictors above inactive ones. We compare five forest-based scores: random forests (RF), one-time and per-tree jittered random forests (RF-one-time, RF-per-tree), UFI \citep{zhou2021unbiased}, and CForest \citep{hothorn2006unbiased}, and two boosting-based scores: XGBoost (XGB) and XGBoost with one-time jittering (XGB-one-time) \citep{chen2016xgboost}. Random forests are fit with \texttt{scikit-learn} \citep{pedregosa2011scikit}, UFI with its authors' implementation, CForest with the R package \texttt{partykit} \citep{hothorn2015partykit}, and XGBoost with \texttt{xgboost}. Each method is implemented with its package's default parameters. 

Ranking quality is summarized by the area under the receiver operating characteristic (ROC) curve (AUC), defined as follows. Let $I_j$ denote the importance of the $j$th predictor. Thresholding at $t$ and comparing the selected set $\{j:I_j\geq t\}$ against $S=S_X\cup S_Z$ gives
\begin{align*}
\mathrm{TPR}(t)=\frac{|\{j\in S:I_j\geq t\}|}{|S|}
\quad\text{and}\quad
\mathrm{FPR}(t)=\frac{|\{j\notin S:I_j\geq t\}|}{p+q-|S|}.
\end{align*}
The ROC curve traces $(\mathrm{FPR}(t),\mathrm{TPR}(t))$ as $t$ varies, and AUC is the area beneath it. Equivalently, AUC is the probability that a uniformly random active predictor outscores a uniformly random inactive one, with $0.5$ corresponding to random ranking and $1$ to perfect separation.

\cref{tab:study1_auc} shows the results. Per-tree and one-time jittering are the best and second best performing methods in all four settings, respectively. RF has substantially lower AUC than the other methods, regularly ranking all 25 continuous predictors above the binary ones (\cref{fig:roc_rf}). This is consistent with the asymptotic mixed data bias results in \cref{sec:theory}, despite the relatively moderate sample size, $n=200$. Jittering has little effect on XGBoost rankings: each jittered AUC is within one standard deviation of its un-jittered counterpart. In the next section, however, we find that jittering XGBoost significantly improves variable selection performance when combined with IPSS.

\begin{table*}[htbp]
\centering
\small
\setlength{\tabcolsep}{7pt}
\begin{tabular}{lcccc}
\toprule
& \multicolumn{2}{c}{Independent covariates}
& \multicolumn{2}{c}{Correlated covariates} \\
\cmidrule(lr){2-3}
\cmidrule(lr){4-5}
Method & Linear & Nonlinear & Linear & Nonlinear \\
\midrule
RF & 0.80 (0.05) & 0.80 (0.05) & 0.71 (0.05) & 0.72 (0.05) \\
RF-one-time & 0.94 (0.05) & 0.89 (0.07) & 0.86 (0.07) & 0.86 (0.06) \\
RF-per-tree & \textbf{0.96} (0.04) & \textbf{0.92} (0.05) & \textbf{0.88} (0.06) & \textbf{0.90} (0.05) \\
UFI & 0.92 (0.05) & 0.87 (0.07) & 0.83 (0.06) & 0.83 (0.07) \\
CForest & 0.92 (0.05) & 0.82 (0.07) & 0.84 (0.08) & 0.82 (0.08) \\
XGB & 0.86 (0.07) & 0.83 (0.06) & 0.76 (0.08) & 0.76 (0.08) \\
XGB-one-time & 0.84 (0.06) & 0.78 (0.08) & 0.78 (0.08) & 0.75 (0.08) \\
\bottomrule
\end{tabular}
\caption{\textit{Feature-ranking performance}. Mean AUC over $100$ trials, with the standard deviation in parentheses; the best AUC in each setting is in bold (higher is better).}
\label{tab:study1_auc}
\end{table*}


\subsection{Variable selection with false discovery control}\label{sec:simulation_ipss}


Ranking predictors by importance generally lacks a rigorous means of selecting variables. In practice, the ``top-$k$'' ranked predictors are often retained with little or no explanation of how $k$ was chosen. Variable selection methods with false discovery control offer a statistically principled alternative, enabling variables to be selected according to a user-specified error rate that is suitable for the application at hand.

In this second study, we investigate whether the benefits of adding noise extend to variable selection with false discovery control. To check this, we combine various importance scores with the nonparametric extension of IPSS \citep{melikechi2025nonparametric}, which accepts an arbitrary importance measure and produces a $q$-value, $q_j$, for each predictor, where $q_j$ is the smallest false discovery rate (FDR) at which the $j$th predictor is selected. For a nominal FDR level $\alpha$, the predictors $\{j:q_j\leq\alpha\}$ should, on average, achieve an FDR of $\alpha$ if the importance scores are well-calibrated.

We combine five importance scores with IPSS: random forest and XGBoost MDI importance, each with and without one-time jittering, and UFI. Per-tree jittering is omitted because its performance is similar to its one-time counterpart; CForest is omitted due to its computational burden and because it underperformed jittering and UFI in \cref{sec:simulation_auc}.

Performance is evaluated over an evenly spaced grid of FDR levels $\alpha\in\{0,0.01,\dots,0.50\}$. Letting $\widehat{S}(\alpha) = \{j : q_j\leq \alpha\}$ denote the selected set at level $\alpha$ and $S=S_X\cup S_Z$ the active set, the false discovery proportion (FDP) and true positive rate (TPR) at $\alpha$ are
\begin{align*}
\mathrm{FDP}(\widehat S(\alpha))=\frac{|\widehat S(\alpha)\setminus S|}{\max(|\widehat S(\alpha)|,\,1)}
\quad\text{and}\quad
\mathrm{TPR}(\widehat S(\alpha))=\frac{|\widehat S(\alpha)\cap S|}{|S|}.
\end{align*}
The empirical FDR at level $\alpha$ is then the mean of $\mathrm{FDP}(\widehat S(\alpha))$ over 100 trials.

\cref{fig:ipss_all} shows the results. IPSS with MDI from jittered XGBoost performs best overall, both controlling the FDR while yielding relatively high TPRs across most nominal FDR values. IPSS with regular XGBoost, by contrast, tends to exceed its target FDR, especially in the correlated settings. For IPSS with standard RF the selection is essentially frozen: beyond the smallest $\alpha$ levels, all five active continuous predictors are recovered, but so too are the $20$ inactive ones. No binary predictors are ever selected, consistent with the theoretical results in \cref{sec:theory}. This is why its TPR plateaus at $5/10=0.5$ and its empirical FDR plateaus at $20/25=0.8$, greatly exceeding the target FDR. IPSS with jittered RF drastically improves FDR control in all four settings, reducing false discoveries while often matching or exceeding standard RF's TPR at small-to-moderate target FDRs. UFI also controls FDR but recovers fewer active predictors than jittered RF in general.

\begin{figure}[htbp]
\centering
\includegraphics[width=\linewidth]{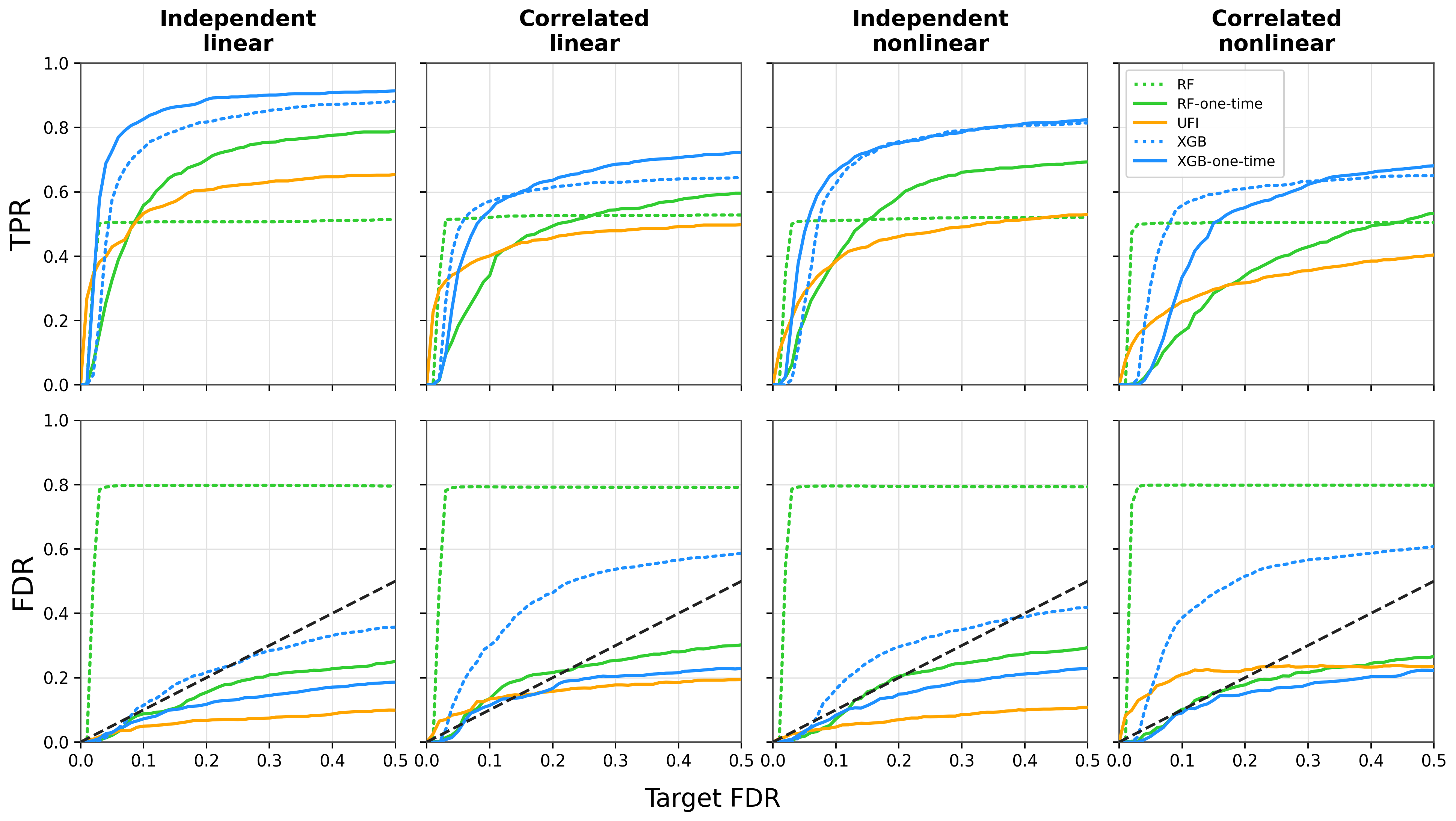}
\caption{\textit{IPSS with different MDI scores}. TPR (top row) and FDR (bottom row) against the target FDR level for IPSS with MDI scores from random forests (RF), one-time jittered random forests (RF-one-time), UFI, XGBoost (XGB), and one-time jittered XGBoost (XGB-one-time). Columns correspond to the four simulation designs. The dashed diagonal in the FDR plots marks exact FDR calibration; methods lying at or below this line successfully control FDR.}
\label{fig:ipss_all}
\end{figure}


\section{Applications}\label{sec:realdata}


In \cref{sec:bladder}, we investigate the effect of jittering on a bladder cancer dataset from The Cancer Genome Atlas \citep[TCGA;][]{weinstein2013cancer}. This is the application that motivated this work: while applying random forests to these data and other datasets like it, we observed that pathologic stage, a known categorical predictor of survival, was often given lower importance than dozens or even hundreds of other features. In \cref{sec:ml_applications}, we report the effect of jittering on seven additional datasets spanning multiple application areas.


\subsection{Bladder cancer}\label{sec:bladder}


We consider the problem of predicting 18-month survival after a diagnosis of bladder urothelial carcinoma using microRNA expression and clinical data, obtained from TCGA via LinkedOmics \citep{vasaikar2018linkedomics}. After cleaning (\cref{supsec:blca}), the data comprise $n=306$ patients and $425$ predictors: $423$ continuous microRNA measurements, patient age, and a categorical predictor, pathologic stage, recorded at three ordered levels (II, III, IV).

Pathologic stage, or just \textit{stage}, measures how much cancer is in the body and how much it has spread. It is an established prognostic factor in bladder cancer and oncology more broadly \citep{shariat2007discrepancy}. Thus, there is strong prior evidence that a reliable importance measure should rank stage highly in these data. To test this, we fit the five importance methods of \cref{sec:simulation_auc} over $50$ random seeds and recorded where stage ranks among the $425$ predictors (\cref{fig:bladder_boxplots}). One-time and per-tree jittering rank it first under nearly every seed, with a median rank of $1$ and an interquartile range (IQR) of $[1,1]$. UFI also ranks stage highly but is more variable than both jittering methods, with median rank $2.5$ and IQR $[1,8]$. The other two methods are far less consistent: RF gives stage a median rank of $36$ with IQR $[15,67]$, and CForest gives a median of $40.5$ with IQR $[25,60]$.


\begin{figure}[htbp]
\centering
\includegraphics[width=\linewidth]{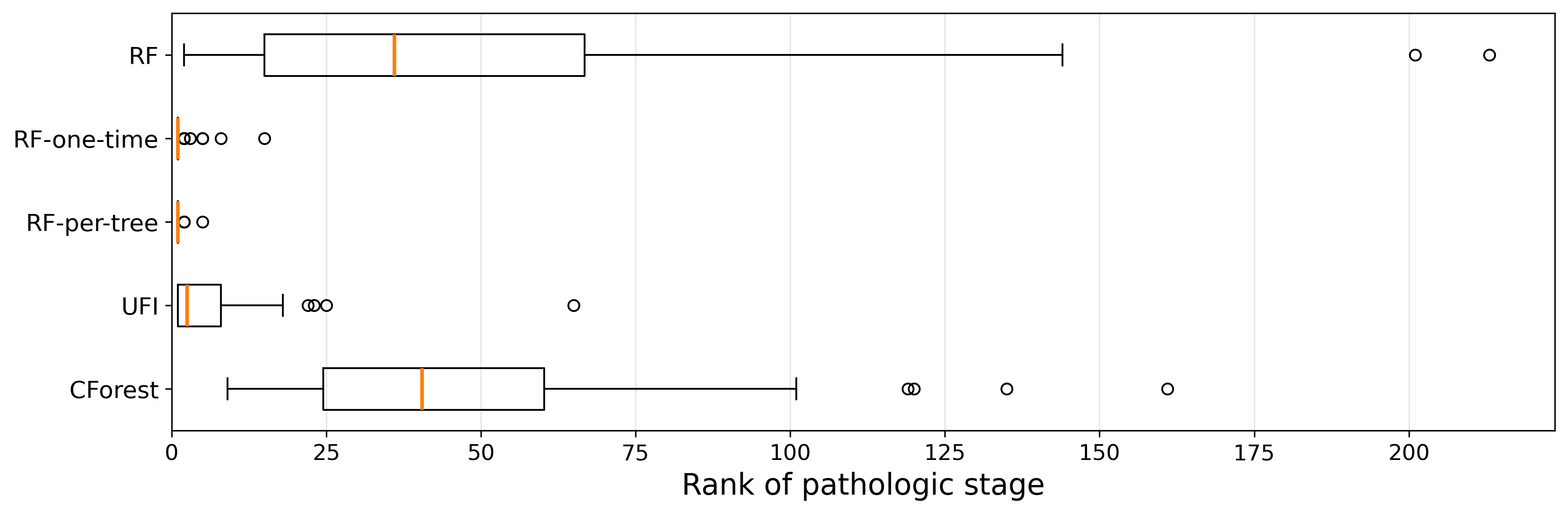}
\caption{\textit{Rank of pathologic stage over 50 random seeds}. Rank of pathologic stage among the $425$ bladder cancer predictors of 18-month survival over $50$ random seeds.}
\label{fig:bladder_boxplots}
\end{figure}


\subsection{Machine learning datasets}\label{sec:ml_applications}


To assess the effect of jittering more broadly, we analyze seven datasets spanning multiple domains, including manufacturing, medicine, finance, and chemistry. Dataset names, sample sizes, and numbers of continuous and categorical predictors are listed in the first four columns of \cref{tab:ipss_null_spikein}. All data are publicly available and can be downloaded from OpenML \citep{vanschoren2014openml} or the Penn Machine Learning Benchmark \citep{Olson2017PMLB}. 

We run two analyses on each dataset. In the first, we add a synthetic null copy of each observed predictor by randomly permuting the latter's rows. This creates a set of variables known to be irrelevant, providing a way to evaluate selection performance on data for which the ground truth is unknown. The second analysis applies IPSS to the original data, as one would in practice, and asks whether the selected predictors are supported by prior evidence.

\textit{Synthetic nulls}. We add one synthetic null predictor per observed predictor by randomly permuting its rows, thereby doubling the dimension of each dataset. The real predictor and its permuted copy have identical marginal distributions, but the latter's association with the response is removed, at least up to sampling variability. \cref{tab:ipss_null_spikein} reports the average number of synthetic nulls selected by IPSS with RF, RF-one-time, XGBoost, and XGBoost-one-time over $10$ trials, where each trial introduces a new synthetic copy of each predictor. The target FDR for IPSS is set to $0.1$ throughout.

IPSS with MDI from RF selects a substantial proportion of synthetic nulls in most cases; IPSS with XGBoost selects far fewer synthetic nulls than RF, but still identifies a nonegligible number on certain data, most notably Hypothyroid. In both cases, most of the erroneously selected synthetic null predictors are continuous. By contrast, IPSS with jittered RF selects almost no synthetic nulls across all datasets. Similar results hold when ranking predictors by raw importance scores, with RF ranking continuous synthetic nulls highly in several of these datasets (\cref{fig:rank_null_spikein}).

\textit{Domain analyses}. IPSS with RF favors continuous and high-cardinality predictors (\cref{fig:ipss_selection_cardinality}). IPSS with jittered RF and UFI instead recovers low-cardinality categorical predictors that RF misses (\cref{supsec:ml}), such as process controls in Cylinder Banding and family history in SAheart, both well-established predictors in their fields. XGBoost and jittered XGBoost also recover many of these predictors, leaving RF as the lone outlier.

\begin{table}[htbp]
\centering
\small
\setlength{\tabcolsep}{6pt}
\renewcommand{\arraystretch}{1.05}
\begin{tabular}{lccccccc}
\toprule
Dataset & $n$ & $p$ & $q$ & RF & RF-one-time & XGB & XGB-one-time \\
\midrule
QSAR Biodegradation & 1055 & 19 & 22 & 14.7 (14.7) & 0.0 (0.0) & 0.6 (0.5) & 0.2 (0.1) \\
Cylinder Banding & 540 & 13 & 22 & 10.7 (9.5) & 0.0 (0.0) & 0.0 (0.0) & 0.0 (0.0) \\
Hypothyroid & 3163 & 7 & 18 & 9.8 (7.0) & 0.2 (0.0) & 2.7 (2.7) & 0.3 (0.0) \\
German Credit & 1000 & 3 & 17 & 6.7 (3.0) & 0.0 (0.0) & 0.5 (0.4) & 0.0 (0.0) \\
Hepatitis & 155 & 6 & 13 & 6.0 (6.0) & 0.0 (0.0) & 0.1 (0.1) & 0.0 (0.0) \\
SAheart & 462 & 8 & 1 & 0.0 (0.0) & 0.0 (0.0) & 0.0 (0.0) & 0.0 (0.0) \\
Titanic & 2207 & 2 & 6 & 4.0 (2.0) & 0.0 (0.0) & 0.0 (0.0) & 0.0 (0.0) \\
\bottomrule
\end{tabular}
\caption{\textit{Number of synthetic null predictors selected}. For each dataset ($n$ samples, $p$ continuous and $q$ categorical predictors), the mean number of synthetic null predictors selected by IPSS with each base selector at target FDR $\alpha=0.1$ over $10$ draws, with the number that are continuous in parentheses.}
\label{tab:ipss_null_spikein}
\end{table}


\section{Discussion}\label{sec:discussion}


MDI scores in mixed data are heavily influenced by the number of unique values each predictor can take. We provide a theoretical analysis of this asymmetry in CART and CART ensembles and explain its preference for continuous and high-cardinality predictors. Although the theory here does not cover additional components of tree ensembles, such as observation resampling and feature subsampling, extensive synthetic and real-data experiments show that the same phenomenon persists in practice. One-time jittering equalizes split opportunities through a simple preprocessing step, offering a simple, computationally efficient means of improving importance rankings and variable selection with false discovery control without compromising predictive accuracy.


\subsection*{Code and data availability}


Code to reproduce the results in this paper, including tables and figures, is available at \url{https://github.com/melikechi-lab/just-add-noise}. All datasets used in this work are public: the bladder cancer cohort is downloaded from LinkedOmics \citep{vasaikar2018linkedomics}, and the remaining datasets are downloaded from OpenML \citep{vanschoren2014openml} or the Penn Machine Learning Benchmark \citep{Olson2017PMLB}.

\putbib[reference]
\end{bibunit}

\clearpage

\title{{\Large S}{\large upplementary material}}
\date{}
\maketitle

\setcounter{page}{1}
\setcounter{section}{0}
\setcounter{table}{0}
\setcounter{figure}{0}
\setcounter{equation}{0}
\setcounter{algorithm}{0}

\renewcommand{\theHsection}{Ssection.\arabic{section}}
\renewcommand{\theHtable}{Stable.\arabic{table}}
\renewcommand{\theHfigure}{Sfigure.\arabic{figure}}
\renewcommand{\theHequation}{Sequation.\arabic{equation}}
\renewcommand{\thepage}{S\arabic{page}}
\renewcommand{\thesection}{S\arabic{section}}
\renewcommand{\thetable}{S\arabic{table}}
\renewcommand{\thefigure}{S\arabic{figure}}
\renewcommand{\theequation}{S\arabic{equation}}
\renewcommand{\thealgorithm}{S\arabic{algorithm}}

\begin{bibunit}[ims]


\section{Proofs for the null model}\label{supsec:null_proofs}



\subsection{Additional setup and notation}\label{supsec:null_setup}


In the proofs of \cref{thrm:null_root,thrm:tree_null}, replacing $Y$ by $(Y-\mathbb{E}(Y))/\sqrt{\operatorname{Var}(Y)}$ preserves split comparisons and normalized importances. We therefore assume without loss of generality that $\mathbb{E}(Y)=0$ and $\operatorname{Var}(Y)=1$. A direct calculation using \eqref{eq:variance_impurity} gives, for any proper subset $A$ of a node $t$,
\begin{equation}
\Delta \mathcal{I}(t, A)
= \frac{1}{n_t}\,\varphi_{t}(A) - \frac{S_t^2}{n_t^2},
\quad\text{where}\quad
\varphi_{t}(A)
=
\frac{S_A^2}{|A|} + \frac{S_{t \setminus A}^2}{|t \setminus A|}.
\label{eq:varphi_identity}
\end{equation}
Since $n_t$ and $S_t$ do not depend on $A$, maximizing $\Delta \mathcal{I}(t, A)$ is equivalent to maximizing $\varphi_{t}(A)$. The quantity $\varphi_t$ is the central object in what follows, both here and in \cref{supsec:nonnull}. When $t = [n]$ is the root, we drop subscripts and write $\varphi_n = \varphi_{[n]}$, $S_n = S_{[n]}$, and so on.


\subsection{Proof of \cref{thrm:null_root}}\label{supsec:null_root}


The proof of \cref{thrm:null_root} reduces to the case $p = q = 1$, which we isolate as a lemma below.

\begin{lemma}\label{lem:pq1}
Suppose $p = 1$, $q = 1$ and that $Y$ is independent of $X = X_1$ and $Z = Z_1$. Then
\begin{align*}
\lim_{n\to\infty} \mathbb{P}\bigl(M_n(X) > M_n(Z)\bigr) = 1.
\end{align*}
\end{lemma}

\begin{proof}[Proof]
By the standardization convention above, $\mathbb{E}(Y)=0$ and $\mathbb{E}(Y^2)=1$. For the continuous predictor, \eqref{eq:varphi_identity} gives
\begin{align*}
\zeta_n
&:=
\max_{1\leq \ell<n}\frac{S_{L_\ell(X)}^2}{\ell},
\\
nM_n(X)
&=
\max_{1\leq \ell<n}\varphi_n\bigl(L_\ell(X)\bigr)
-
\frac{S_n^2}{n}
\geq
\zeta_n-\frac{S_n^2}{n}.
\end{align*}
Let $\sigma$ be the permutation of $[n]$ induced by sorting the $x_i$'s in increasing order. $Y\perp X$ implies $Y\perp\sigma$; consequently $(y_{\sigma(1)},\ldots,y_{\sigma(n)})$ is distributed as $(y_1,\ldots,y_n)$. Writing $\tilde S_\ell = \sum_{i=1}^\ell y_{\sigma(i)}$, we have $S_{L_\ell(X)} = \tilde S_\ell$, so $\zeta_n\stackrel{d}{=}\xi_n := \max_{1\leq k<n}S_k^2/k$, where $S_k = \sum_{i=1}^k Y_i$.
Since $S_n^2/n \leq \xi_{n+1}$ for every $n$, the law of the iterated logarithm applied to $\{Y_i\}$ gives
\begin{equation}
\limsup_{n\to\infty}\frac{\xi_{n+1}}{2\log\log n}
\geq
\limsup_{n\to\infty}\frac{S_n^2}{2n\log\log n}
= 1 \quad \text{almost surely.}
\label{eq:lil}
\end{equation}
Since $\log\log n \to \infty$ and $\xi_n$ is nondecreasing, \eqref{eq:lil} gives $\xi_n \to \infty$ almost surely. Moreover, $\mathbb{E}[S_n^2/n]=1$, so $S_n^2/n=O_p(1)$. Thus $nM_n(X)\overset{p}{\longrightarrow}\infty$.

For the categorical predictor, let $\mathcal{B}_n(Z)=\{B\in\mathcal{P}^\ast(\mathcal{Z}):0<|L_B(Z)|<n\}$, with the maximum over $\mathcal{B}_n(Z)$ defined to be zero if $\mathcal{B}_n(Z)$ is empty. Conditional on $z_1,\ldots,z_n$, and for any $B\in\mathcal{B}_n(Z)$, let $n_B=|L_B(Z)|$ and $n_B^c=n-n_B$.
Since $Y\perp Z$, $L_B(Z)$ is fixed conditionally on $z_1,\ldots,z_n$, and
\begin{align*}
\mathbb{E}\left[
\frac{S_{L_B(Z)}^2}{n_B}
\middle|
z_1,\ldots,z_n
\right]
=1,
\qquad
\mathbb{E}\left[
\frac{S_{L_B(Z)^c}^2}{n_B^c}
\middle|
z_1,\ldots,z_n
\right]
=1.
\end{align*}
Hence $\mathbb{E}\left[\varphi_n(L_B(Z))\middle|z_1,\ldots,z_n\right]=2$.
Complementary category subsets give the same split and the same value of $\varphi_n$, so after retaining one representative from each pair there are at most $2^{|\mathcal{Z}|-1}-1$ valid categorical splits. Therefore,
\begin{align*}
\mathbb{E}\left[
\max_{B\in\mathcal{B}_n(Z)}
\varphi_n(L_B(Z))
\middle|
z_1,\ldots,z_n
\right]
\le
2\bigl(2^{|\mathcal{Z}|-1}-1\bigr).
\end{align*}
Markov's inequality gives $\max_{B\in\mathcal{B}_n(Z)}\varphi_n(L_B(Z))=O_p(1)$.
By \eqref{eq:varphi_identity},
\begin{align*}
0\leq nM_n(Z)
\leq
\max_{B\in\mathcal{B}_n(Z)}
\varphi_n\bigl(L_B(Z)\bigr)
=O_p(1).
\end{align*}
Since $nM_n(X)\overset{p}{\longrightarrow}\infty$ and $nM_n(Z)=O_p(1)$, the result follows.
\end{proof}

\begin{proof}[Proof of \cref{thrm:null_root}]
For any $k\in[q]$, \cref{lem:pq1} applied to the pair $(X_1,Z_k)$ gives $\mathbb{P}\{M_n(X_1)\leq M_n(Z_k)\}\to0$. Hence
\begin{align*}
&\mathbb P\!\left(
\max_{j\in[p]}M_n(X_j)
\leq
\max_{k\in[q]}M_n(Z_k)
\right)
\leq
\sum_{k=1}^q
\mathbb P\!\left(M_n(X_1)\leq M_n(Z_k)\right)
\overset{p}{\longrightarrow} 0. \qedhere
\end{align*}
\end{proof}


\subsection{Proof of \cref{thrm:tree_null}}\label{supsec:tree_null}


We assume the joint null model throughout this section.

Let $\mathcal{R}$ be the collection of sets
\begin{align*}
R=
\left(\prod_{j=1}^p(a_j,b_j]\right)\times B,
\qquad
-\infty\leq a_j<b_j\leq\infty,
\qquad
B\subseteq\mathcal{Z},
\end{align*}
together with the empty set. This collection contains all CART node regions and candidate split regions and is closed under intersections. For $A\in\mathcal{R}$, define
\begin{align*}
P(A)
&=\mathbb{P}\bigl((X,Z)\in A\bigr),
\\
P_n(A)
&=\frac1n\sum_{i=1}^n
\mathbf{1}\bigl((x_i,z_i)\in A\bigr),
\\
H_n(A)
&=\frac1{\sqrt n}\sum_{i=1}^n
y_i\mathbf{1}\bigl((x_i,z_i)\in A\bigr).
\end{align*}
Only splits with two nonempty children are considered, and a maximum over an empty collection of candidate splits is defined to be zero.

\begin{lemma}[Uniform control over node regions]
\label{lem:tree_uniform}
Under the joint null model,
\begin{align}
\sup_{R\in\mathcal{R}}|P_n(R)-P(R)|
&\overset{p}{\longrightarrow}0,
\label{eq:tree_uniform_counts}
\\
\sup_{R\in\mathcal{R}}|H_n(R)|
&=O_p(1).
\label{eq:tree_uniform_sums}
\end{align}
Moreover, for any data-dependent sequence $R_n\in\mathcal{R}$ such that $P_n(R_n)\overset{p}{\longrightarrow}1$,
\begin{align}
\sup_{A\in\mathcal{R}}
\bigl|H_n(R_n\cap A)-H_n(A)\bigr|
\overset{p}{\longrightarrow}0.
\label{eq:tree_negligible_removal}
\end{align}
\end{lemma}

\begin{proof}[Proof]
Because $p$ and $\mathcal{Z}$ are fixed, $\mathcal{R}$ is a Vapnik--Chervonenkis class. The uniform law of large numbers therefore gives \eqref{eq:tree_uniform_counts}. Consider the weighted indicator class
\begin{align*}
\mathcal{F}
=
\left\{
(x,z,y)\mapsto y\mathbf{1}\bigl((x,z)\in R\bigr):
R\in\mathcal{R}
\right\}.
\end{align*}
This class inherits a polynomial uniform covering-number bound from $\mathcal{R}$ and has square-integrable envelope $|Y|$. It is therefore Donsker by the uniform entropy theorem \citep{vandervaart1996weak}. Under the joint null, $\mathbb{E}\left[Y\mathbf{1}\bigl((X,Z)\in R\bigr)\right]=0$ for every $R\in\mathcal{R}$.
Thus $H_n$ is the centered empirical process indexed by $\mathcal{F}$, and its asymptotic tightness gives \eqref{eq:tree_uniform_sums}.

Now suppose $P_n(R_n)\overset{p}{\longrightarrow}1$. By \eqref{eq:tree_uniform_counts}, $P(R_n^c)\overset{p}{\longrightarrow}0$. For deterministic $R,A\in\mathcal{R}$, independence of $Y$ and $(X,Z)$ gives
\begin{align*}
\mathbb{E}\left[
\left\{
Y\mathbf{1}\bigl((X,Z)\in R\cap A\bigr)
-
Y\mathbf{1}\bigl((X,Z)\in A\bigr)
\right\}^2
\right]
&=P(A\setminus R)
\leq P(R^c).
\end{align*}
Fix $\varepsilon>0$. On the event $\{P(R_n^c)<\varepsilon^2\}$, every pair $f_{R_n\cap A},f_A$ with $A\in\mathcal{R}$ has $L_2(P)$-distance at most $\varepsilon$, so the left-hand side of \eqref{eq:tree_negligible_removal} is bounded by the oscillation of the empirical process over $\varepsilon$-balls. Asymptotic equicontinuity of the Donsker process, together with $P(R_n^c)\overset{p}{\longrightarrow}0$, then gives \eqref{eq:tree_negligible_removal} on letting $\varepsilon\to0$. In particular, $R_n$ may depend on the responses; no conditional independence within the selected region is required.
\end{proof}

\begin{lemma}[Nodewise dominance and unbalanced splits]
\label{lem:tree_nodewise}
Let $R_n\in\mathcal{R}$ be data-dependent and define $t=\bigl\{i\in[n]:(x_i,z_i)\in R_n\bigr\}$.
Suppose $n_t/n\overset{p}{\longrightarrow}1$. Then
\begin{align}
nM_t(X_1)
&\overset{p}{\longrightarrow}\infty,
\label{eq:tree_cont_divergence}
\\
n\max_{k\in[q]}M_t(Z_k)
&=O_p(1).
\label{eq:tree_cat_tightness}
\end{align}
Moreover, for every $\eta\in(0,1/2)$,
\begin{align}
n\max_{\substack{
t_L\text{ a candidate child of }t\\
\eta n_t\leq |t_L|\leq(1-\eta)n_t
}}
\Delta\mathcal{I}(t,t_L)
=O_p(1).
\label{eq:tree_balanced_tightness}
\end{align}
Consequently, a maximizing CART split of $t$ is continuous with probability tending to one, and its children satisfy
\begin{align*}
\frac{\min\{|t_L|,|t_R|\}}{n_t}
\overset{p}{\longrightarrow}0.
\end{align*}
\end{lemma}

\begin{proof}[Proof]
For a candidate split region $A$, write $t_A=\bigl\{i\in t:(x_i,z_i)\in A\bigr\}$, $\alpha_n=P_n(R_n)=n_t/n$, and $\beta_n(A)=P_n(R_n\cap A)=|t_A|/n$.
By \eqref{eq:varphi_identity}, whenever $t_A$ is a proper subset of $t$,
\begin{align}
n\Delta\mathcal{I}(t,t_A)
=
\frac{
\left\{
H_n(R_n\cap A)
-
\frac{\beta_n(A)}{\alpha_n}H_n(R_n)
\right\}^2
}{
\beta_n(A)\{\alpha_n-\beta_n(A)\}
}.
\label{eq:tree_scaled_gain}
\end{align}

\paragraph{Categorical gains.}
For a valid categorical split region $A=\bigl\{(x,z):z_k\in B\bigr\}$, $B\in\mathcal{P}^\ast(\mathcal{Z}_k)$, with $0<P(A)<1$, we have $\alpha_n\overset{p}{\longrightarrow}1$ and $\beta_n(A)\overset{p}{\longrightarrow}P(A)$.
Indeed,
$|\beta_n(A)-P_n(A)|\leq 1-P_n(R_n)\overset{p}{\longrightarrow}0$.
The denominator in \eqref{eq:tree_scaled_gain} therefore converges to $P(A)\{1-P(A)\}>0$, while its numerator is $O_p(1)$ by \eqref{eq:tree_uniform_sums}. There are only finitely many categorical split regions. Those with population probability zero or one cannot produce two nonempty children, almost surely. This proves \eqref{eq:tree_cat_tightness}.

\paragraph{Balanced gains.}
If $\eta\leq |t_A|/n_t\leq1-\eta$, then $\beta_n(A)\{\alpha_n-\beta_n(A)\}\geq \alpha_n^2\eta(1-\eta)$.
Since $\alpha_n\overset{p}{\longrightarrow}1$ and the numerator in \eqref{eq:tree_scaled_gain} is uniformly $O_p(1)$ by \eqref{eq:tree_uniform_sums}, we obtain \eqref{eq:tree_balanced_tightness}.

\paragraph{Continuous gains.}
Let $F(s)=\mathbb{P}(X_1\leq s)$. For $u\in(0,1)$, define $s_u=\inf\{s:F(s)\geq u\}$ and $A_u=\bigl\{(x,z):x_1\leq s_u\bigr\}$.
Continuity of $F$ gives $P(A_u)=u$. Since the CART candidate thresholds along $X_1$ realize every split of the order statistics of $\{x_{i1}:i\in t\}$ into a lower and an upper block, $t_{A_u}$ is the lower child of an actual candidate split of $t$; both of its children are nonempty with probability tending to one because $u\in(0,1)$ and $n_t/n\overset{p}{\longrightarrow}1$. By \eqref{eq:tree_negligible_removal},
\begin{align*}
H_n(R_n\cap A_u)&=H_n(A_u)+o_p(1),
\\
H_n(R_n)&=H_n(\mathbb{R}^p\times\mathcal{Z})+o_p(1).
\end{align*}
Also, $\beta_n(A_u)\overset{p}{\longrightarrow}u$. For any finite collection $u_1,\ldots,u_m\in(0,1)$, the multivariate central limit theorem and \eqref{eq:tree_scaled_gain} therefore give
\begin{align*}
\left(
n\Delta\mathcal{I}(t,t_{A_{u_\ell}})
\right)_{\ell=1}^m
\overset{d}{\longrightarrow}
\left(
\frac{\mathbb{B}(u_\ell)^2}{u_\ell(1-u_\ell)}
\right)_{\ell=1}^m,
\end{align*}
where $\mathbb{B}$ is a standard Brownian bridge and $\overset{d}{\longrightarrow}$ denotes convergence in distribution.

By the law of the iterated logarithm at zero, $\sup_{0<u<1}\mathbb{B}(u)^2/\{u(1-u)\}=\infty$ almost surely.
Consequently, for every $K>0$ and $\varepsilon>0$, continuity of the Brownian bridge on $(0,1)$ allows us to choose a finite collection $u_1,\ldots,u_m$ such that
\begin{align*}
\mathbb{P}\left(
\max_{1\leq\ell\leq m}
\frac{\mathbb{B}(u_\ell)^2}{u_\ell(1-u_\ell)}
>K
\right)>1-\varepsilon.
\end{align*}
Since $M_t(X_1)$ maximizes over all threshold splits,
\begin{align*}
\liminf_{n\to\infty}
\mathbb{P}\bigl(nM_t(X_1)>K\bigr)
\geq1-\varepsilon.
\end{align*}
As $\varepsilon>0$ is arbitrary, this proves \eqref{eq:tree_cont_divergence}. Combining \eqref{eq:tree_cont_divergence} and \eqref{eq:tree_cat_tightness} gives
\begin{align*}
\mathbb{P}\left(
\max_{k\in[q]}M_t(Z_k)
\geq
\max_{j\in[p]}M_t(X_j)
\right)\longrightarrow0.
\end{align*}
Likewise, \eqref{eq:tree_balanced_tightness} shows that, for every fixed $\eta\in(0,1/2)$, the probability that a maximizing split has both children of size at least $\eta n_t$ tends to zero. This proves the final assertion.
\end{proof}

\begin{proof}[Proof of \cref{thrm:tree_null}]
Fix $T\in\mathcal{T}$ and let $D$ be a common upper bound on the tree depths. Starting from $t_0=[n]$, follow the larger child whenever the current node splits, breaking size ties arbitrarily. Once a terminal node is reached, retain that node at subsequent steps.

We prove inductively that $n_{t_d}/n\overset{p}{\longrightarrow}1$ for every fixed $d\leq D$.
This holds at $d=0$. Suppose it holds at depth $d$. The region corresponding to $t_d$ belongs to $\mathcal{R}$, so \cref{lem:tree_nodewise} applies despite the response-dependent selection of that region. If $t_d$ splits, its smaller child has $o_p(n)$ observations, and its larger child therefore contains $n-o_p(n)$ observations. If $t_d$ is terminal, its index set is unchanged. This completes the induction.

The same lemma shows that, at each of the finitely many steps, the probability that $t_d$ splits along a categorical predictor tends to zero. Moreover, whenever $t_d$ splits, its smaller child has fewer than $cn$ observations with probability tending to one. The node-size condition in \cref{thrm:tree_null} then forces that child to be terminal. Thus, with probability tending to one, all internal nodes of $T$ lie on the larger-child path, and every split on that path is continuous. Taking a union bound over the finitely many trees $T\in\mathcal{T}$, the at most $D$ depths, and the two failure events at each step (the smaller child is not terminal; the split is categorical) gives
\begin{align*}
\lim_{n\to\infty}
\mathbb{P}\!\left(
\sum_{T\in\mathcal{T}}\sum_{t\in T}\sum_{k=1}^q
\mathbf{1}\bigl(t\text{ splits along }Z_k\bigr)>0
\right)=0.
\end{align*}

By \cref{lem:tree_nodewise} applied at the root, the root's maximizing gain is positive with probability tending to one. Hence the total MDI is positive with probability tending to one. On the event that the total MDI is positive and no categorical split occurs, the categorical ratio in \cref{thrm:tree_null} equals zero and the combined continuous ratio equals one. This event has probability tending to one, proving the claimed convergences in probability.
\end{proof}


\subsection{Proof of continuous-only fairness result (\cref{prop:fair})}\label{supsec:fair}


\begin{proof}
For a permutation $\pi$ of $[p]$, let $X^{\pi}$ denote the data with the $p$ coordinates of every sample relabeled by $\pi$. Since $X_1,\ldots,X_p$ are exchangeable and $Y\perp X$, we have $(X^{\pi},Y)\stackrel{d}{=}(X,Y)$. By symmetry of the CART construction, relabeling the predictors relabels the raw importance vector, so
\begin{align*}
\bigl(\Phi_n(X_1),\ldots,\Phi_n(X_p)\bigr)
\stackrel{d}{=}
\bigl(\Phi_n(X_{\pi(1)}),\ldots,\Phi_n(X_{\pi(p)})\bigr)
\qquad\text{for every permutation }\pi.
\end{align*}
Thus the raw importance vector is exchangeable. Dividing by the positive coordinate sum preserves exchangeability, and the normalized coordinates sum to one. Their expectations are therefore all equal to some $a$, with $pa=1$. Hence the desired expectation is $1/p$ for every $j\in[p]$.
\end{proof}


\section{Proofs for the non-null model}\label{supsec:nonnull}



\subsection{Additional notation}
\label{supsec:nonnull_notation}


We first introduce the population and empirical quantities used in the proof of \cref{thrm:nonnull_allocation}. For a threshold $c\in\mathbb R$, define $F(c)=\mathbb P(X\le c)$, $\mu=\mathbb E[Y]$, and $A(c)=\mathbb E\bigl[(Y-\mu)\mathbf 1(X\le c)\bigr]$. For every threshold satisfying $0<F(c)<1$, define the corresponding population variance-impurity decrease by $G(c)=A(c)^2/\{F(c)(1-F(c))\}$.
Equivalently, $G(c)=F(c)\bigl(1-F(c)\bigr)\left[\mathbb E(Y\mid X\le c)-\mathbb E(Y\mid X>c)\right]^2$.
The largest population gain obtainable from a threshold split along $X$ is $M_*(X)=\sup_{c:\,0<F(c)<1}G(c)$. At the sample level, define $\widehat F_n(c)=n^{-1}\sum_{i=1}^n\mathbf 1(x_i\le c)$ and $\widehat A_n(c)=n^{-1}\sum_{i=1}^n(y_i-\bar Y_n)\mathbf 1(x_i\le c)$. For thresholds satisfying $0<\widehat F_n(c)<1$, define $\widehat G_n(c)=\widehat A_n(c)^2/\{\widehat F_n(c)(1-\widehat F_n(c))\}$. The largest empirical gain along $X$ is then $M_n(X)=\sup_{c:\,0<\widehat F_n(c)<1}\widehat G_n(c)$.
For the binary categorical predictor $Z$, let $M_n(Z)$ denote the empirical variance-impurity decrease obtained by separating the observations with $Z=z_0$ from those with $Z=z_1$. 


\subsection{Limit of the categorical gain}


\begin{lemma}
\label{lem:cat_gain} Set $v_Z=\operatorname{Var}(g(Z))$. Under the additive model, $M_n(Z)\xrightarrow{a.s.}v_Z$.
\end{lemma}

\begin{proof}
The binary variable $Z$ has only one nontrivial split, separating $Z=z_0$ from $Z=z_1$. Let $\widehat\pi_n=n^{-1}\sum_{i=1}^n\mathbf 1(z_i=z_0)$, and let $\bar Y_{n,z}$ be the sample mean of $Y$ among observations with $Z=z$. The variance-impurity decrease for the split on $Z$ is
\begin{align*}
M_n(Z)
=
\widehat\pi_n(1-\widehat\pi_n)
(\bar Y_{n,z_0}-\bar Y_{n,z_1})^2.
\end{align*}
By the strong law of large numbers, $\widehat\pi_n\to \pi$ a.s. and $\bar Y_{n,z}\to\mathbb E[Y\mid Z=z]$ a.s. Since $X$, $Z$, and $\varepsilon$ are mutually independent, $\mathbb E[Y\mid Z=z]=\mathbb E[f(X)]+g(z)$. Therefore $\bar Y_{n,z_0}-\bar Y_{n,z_1}\to g(z_0)-g(z_1)$ a.s.
It follows that
\begin{align*}
&M_n(Z)
\to
\pi(1-\pi)\{g(z_0)-g(z_1)\}^2
=
\operatorname{Var}(g(Z))
=
v_Z
\qquad\text{a.s.} \qedhere
\end{align*}
\end{proof}


\subsection{Total impurity decrease in a fully grown tree}

\begin{lemma}[Total impurity decrease in a fully grown tree]
\label{lem:total_impurity_decrease} Consider a fully grown CART tree, meaning that every leaf has zero impurity. Then
\begin{align*}
\sum_{t\text{ internal}}
\frac{n_t}{n}\Delta\mathcal I(t,t_L)
=
\mathcal I([n]).
\end{align*}
Consequently, for the unnormalized MDI importances,
\begin{align*}
\widehat\Phi_n(X)+\widehat\Phi_n(Z)
=
\mathcal I([n]).
\end{align*}
\end{lemma}

\begin{proof}
By the definition of impurity decrease, for every internal node $t$,
\begin{align*}
\Delta\mathcal I(t,t_L)
=
\mathcal I(t)
-
\frac{|t_L|}{n_t}\mathcal I(t_L)
-
\frac{|t_R|}{n_t}\mathcal I(t_R).
\end{align*}
Multiplying both sides by $n_t/n$ gives
\begin{align*}
\frac{n_t}{n}\Delta\mathcal I(t,t_L)
=
\frac{n_t}{n}\mathcal I(t)
-
\frac{|t_L|}{n}\mathcal I(t_L)
-
\frac{|t_R|}{n}\mathcal I(t_R).
\end{align*}
Now sum this identity over all internal nodes. Every non-root internal node appears once with a positive sign, as a parent node, and once with a negative sign, as a child node. Hence all non-root internal node terms cancel. We get
\begin{align*}
\sum_{t\text{ internal}}
\frac{n_t}{n}\Delta\mathcal I(t,t_L)
=
\mathcal I([n])
-
\sum_{\ell\text{ leaf}}
\frac{|\ell|}{n}\mathcal I(\ell).
\end{align*}
Because every leaf has zero impurity, the leaf sum vanishes. Hence
\begin{align*}
\sum_{t\text{ internal}}
\frac{n_t}{n}\Delta\mathcal I(t,t_L)
=
\mathcal I([n]).
\end{align*}
Finally, by definition, the unnormalized MDI importance of a variable is the sum of the weighted impurity decreases over all internal nodes split on that variable. Since in the present setting the only variables are $X$ and $Z$, summing their importances gives the total weighted impurity decrease:
\begin{align*}
&\widehat\Phi_n(X)+\widehat\Phi_n(Z)
=
\sum_{t\text{ internal}}
\frac{n_t}{n}\Delta\mathcal I(t,t_L)
=
\mathcal I([n]). \qedhere
\end{align*}
\end{proof}


\subsection{Uniform consistency of the gain ingredients}


\begin{lemma}
\label{lem:uniform_ingredients}
Let $\{(X_i,Y_i)\}_{i=1}^n$ be i.i.d., and assume $\mathbb E|Y|<\infty$. Then $\sup_c|\widehat F_n(c)-F(c)|\to0$ a.s. and $\sup_c|\widehat A_n(c)-A(c)|\to0$ a.s.
\end{lemma}

\begin{proof}
The first statement is the Glivenko--Cantelli theorem.

For the second statement, define $H(c)=\mathbb E[Y\mathbf 1(X\le c)]$ and $\widehat H_n(c)=n^{-1}\sum_{i=1}^n y_i\mathbf 1(x_i\le c)$. It is enough to prove $\sup_c|\widehat H_n(c)-H(c)|\to0$ a.s.
Indeed,
$\widehat A_n(c)=\widehat H_n(c)-\bar Y_n\widehat F_n(c)$ and $A(c)=H(c)-\mu F(c)$, so
\begin{align*}
\sup_c|\widehat A_n(c)-A(c)|
\le
\sup_c|\widehat H_n(c)-H(c)|
+
|\bar Y_n-\mu|
+
|\mu|\sup_c|\widehat F_n(c)-F(c)|.
\end{align*}
The strong law of large numbers gives $\bar Y_n\to\mu$ a.s., so it remains to control $\widehat H_n$.

Fix $\varepsilon>0$. Since $\mathbb E|Y|<\infty$, choose
\begin{align*}
-\infty=a_0<a_1<\cdots<a_m=\infty,
\qquad
\max_{1\le j\le m}\mathbb E\left[|Y|\mathbf 1(a_{j-1}<X\le a_j)\right]<\frac{\varepsilon}{3}.
\end{align*}
For $c\in(a_{j-1},a_j]$,
\begin{align*}
\begin{aligned}
|\widehat H_n(c)-H(c)|
&\le
|\widehat H_n(a_{j-1})-H(a_{j-1})|
+
|\widehat H_n(c)-\widehat H_n(a_{j-1})|
+
|H(c)-H(a_{j-1})|.
\end{aligned}
\end{align*}
Moreover,
\begin{align*}
|\widehat H_n(c)-\widehat H_n(a_{j-1})|
&\le
\frac1n\sum_{i=1}^n |y_i|\mathbf 1(a_{j-1}<x_i\le a_j),
\\
|H(c)-H(a_{j-1})|
&\le
\mathbb E\left[|Y|\mathbf 1(a_{j-1}<X\le a_j)\right].
\end{align*}
Taking suprema over $c$, we obtain
\begin{align*}
\begin{aligned}
\sup_c|\widehat H_n(c)-H(c)|
&\le
\max_{0\le j\le m}
|\widehat H_n(a_j)-H(a_j)| \\
&\quad+
\max_{1\le j\le m}
\frac1n\sum_{i=1}^n
|y_i|\mathbf 1(a_{j-1}<x_i\le a_j) \\
&\quad+
\max_{1\le j\le m}
\mathbb E[
|Y|\mathbf 1(a_{j-1}<X\le a_j)
].
\end{aligned}
\end{align*}
For each fixed $j$, the strong law of large numbers gives $\widehat H_n(a_j)\to H(a_j)$ a.s. and $n^{-1}\sum_{i=1}^n |y_i|\mathbf 1(a_{j-1}<x_i\le a_j)\to \mathbb E[|Y|\mathbf 1(a_{j-1}<X\le a_j)]$ a.s.
Since there are only finitely many intervals,
\begin{align*}
\limsup_n\sup_c|\widehat H_n(c)-H(c)|
\le
\frac{2\varepsilon}{3}
<
\varepsilon
\qquad\text{a.s.}
\end{align*}
As $\varepsilon>0$ is arbitrary, $\sup_c|\widehat H_n(c)-H(c)|\to0$ a.s.
The desired conclusion for $\widehat A_n$ follows.
\end{proof}


\subsection{Limit of the continuous gain}


\begin{lemma}
\label{lem:cont_gain} Suppose that $\{(X_i,Y_i)\}_{i=1}^n$ are i.i.d.,
$X$ is absolutely continuous, and
$\mathbb E[Y^2]<\infty$. Then $M_n(X)\xrightarrow{a.s.}M_*(X)$.
\end{lemma}

\begin{proof}
We first record the boundary behavior of the population gain. By Cauchy--Schwarz,
\begin{align*}
A(c)^2
=
\left\{
\mathbb E[(Y-\mu)\mathbf 1(X\le c)]
\right\}^2
\le
F(c)\,
\mathbb E[(Y-\mu)^2\mathbf 1(X\le c)].
\end{align*}
Hence
\begin{align*}
G(c)
=
\frac{A(c)^2}{F(c)\{1-F(c)\}}
\le
\frac{
\mathbb E[(Y-\mu)^2\mathbf 1(X\le c)]
}
{1-F(c)}.
\end{align*}
As $F(c)\to0$, the numerator tends to zero by dominated convergence, while $1-F(c)\to1$. Therefore $G(c)\to0$ as $F(c)\to0$. Similarly, since $A(c)=-\mathbb E[(Y-\mu)\mathbf 1(X>c)]$,
we have
\begin{align*}
A(c)^2
\le
\{1-F(c)\}
\mathbb E[(Y-\mu)^2\mathbf 1(X>c)],
\end{align*}
and therefore
\begin{align*}
G(c)
\le
\frac{
\mathbb E[(Y-\mu)^2\mathbf 1(X>c)]
}
{F(c)}
\to0
\qquad\text{as }F(c)\to1.
\end{align*}
Thus the population gain is negligible at the two boundaries.

Write $q(a,u)=a^2/\{u(1-u)\}$, so that $G(c)=q(A(c),F(c))$ and $\widehat G_n(c)=q(\widehat A_n(c),\widehat F_n(c))$ for every threshold $c$. By the law of total variance, $0\leq G(c)\leq\operatorname{Var}(Y)$ whenever $0<F(c)<1$. For each $m\in\mathbb N$, choose a deterministic threshold $c_m$ such that $0<F(c_m)<1$ and $G(c_m)>M_*(X)-1/m$.
By \cref{lem:uniform_ingredients}, $\widehat F_n(c_m)\to F(c_m)$ and $\widehat A_n(c_m)\to A(c_m)$ almost surely. Hence $\widehat G_n(c_m)\to G(c_m)$ almost surely, and $c_m$ induces a valid empirical split for all sufficiently large $n$, almost surely. Thus
\begin{align*}
\liminf_{n\to\infty}M_n(X)\geq G(c_m)
\qquad\text{almost surely}.
\end{align*}
Intersecting these probability-one events over $m\in\mathbb N$ and letting $m\to\infty$ gives
\begin{align*}
\liminf_{n\to\infty}M_n(X)\geq M_*(X)
\qquad\text{almost surely}.
\end{align*}

It remains to prove the upper bound. Fix $\varepsilon>0$. From the boundary behavior above, choose $\alpha'\in(0,1/2)$ such that
\begin{align*}
2\mathbb E\left[
(Y-\mu)^2\mathbf 1(X\le F^{-1}(\alpha'))
\right]
<
\frac{\varepsilon}{2},
\end{align*}
and
\begin{align*}
2\mathbb E\left[
(Y-\mu)^2\mathbf 1(X>F^{-1}(1-\alpha'))
\right]
<
\frac{\varepsilon}{2}.
\end{align*}
Now choose $\alpha\in(0,\alpha')$. We split the thresholds into the three regions $\{\widehat F_n(c)<\alpha\}$, $\{\alpha\le \widehat F_n(c)\le 1-\alpha\}$, and $\{\widehat F_n(c)>1-\alpha\}$. First consider the left boundary, where $\widehat F_n(c)=k/n\le\alpha$.
Then $1-\widehat F_n(c)\ge1/2$, and hence
\begin{align*}
\widehat G_n(c)
\le
\frac{2\widehat A_n(c)^2}{\widehat F_n(c)}.
\end{align*}
Since $\widehat A_n(c)=n^{-1}\sum_{x_i\le c}(y_i-\bar Y_n)$, we obtain
\begin{align*}
\frac{2\widehat A_n(c)^2}{\widehat F_n(c)}
=
\frac{
2\left(\sum_{x_i\le c}(y_i-\bar Y_n)\right)^2
}{nk}.
\end{align*}
By Cauchy--Schwarz,
\begin{align*}
\left(\sum_{x_i\le c}(y_i-\bar Y_n)\right)^2
\le
k\sum_{x_i\le c}(y_i-\bar Y_n)^2.
\end{align*}
Therefore
\begin{align*}
\widehat G_n(c)
\le
\frac2n
\sum_{x_i\le c}(y_i-\bar Y_n)^2.
\end{align*}
The right-hand side is nondecreasing in $c$. Thus
\begin{align*}
\sup_{\widehat F_n(c)\le\alpha}\widehat G_n(c)
\le
\frac2n
\sum_{x_i\le x_{(\lfloor \alpha n\rfloor)}}
(y_i-\bar Y_n)^2.
\end{align*}
Since $\alpha<\alpha'$ and $\sup_c|\widehat F_n(c)-F(c)|\to0$ a.s., we have $x_{(\lfloor \alpha n\rfloor)}\le F^{-1}(\alpha')$ for all sufficiently large $n$ a.s.
Consequently,
\begin{align*}
\sup_{\widehat F_n(c)\le\alpha}\widehat G_n(c)
\le
\frac2n
\sum_{x_i\le F^{-1}(\alpha')}
(y_i-\bar Y_n)^2
\end{align*}
for all large $n$ almost surely.

At the fixed threshold $F^{-1}(\alpha')$, the strong law of large numbers gives
\begin{align*}
\frac2n
\sum_{x_i\le F^{-1}(\alpha')}
(y_i-\bar Y_n)^2
\to
2\mathbb E\left[
(Y-\mu)^2\mathbf 1(X\le F^{-1}(\alpha'))
\right]
<
\frac{\varepsilon}{2}
\qquad\text{a.s.}
\end{align*}
Therefore
\begin{align*}
\limsup_n
\sup_{\widehat F_n(c)\le\alpha}\widehat G_n(c)
\le
\frac{\varepsilon}{2}
\qquad\text{a.s.}
\end{align*}
The right boundary is analogous. Using $\sum_{x_i\le c}(y_i-\bar Y_n)=-\sum_{x_i>c}(y_i-\bar Y_n)$, one obtains
\begin{align*}
\limsup_n
\sup_{\widehat F_n(c)\ge1-\alpha}\widehat G_n(c)
\le
2\mathbb E\left[
(Y-\mu)^2\mathbf 1(X>F^{-1}(1-\alpha'))
\right]
<
\frac{\varepsilon}{2}
\qquad\text{a.s.}
\end{align*}
It remains to control the interior region $\alpha\le \widehat F_n(c)\le 1-\alpha$. Recall $q(a,u)=a^2/\{u(1-u)\}$ from above. On the interior region, the uniform convergence of $\widehat F_n$ gives $\alpha/2\le F(c)\le 1-\alpha/2$ for all large $n$ a.s. Also, $|A(c)|\le \mathbb E|Y-\mu|=:B_A<\infty$, and by \cref{lem:uniform_ingredients}, $\sup_c|\widehat A_n(c)-A(c)|\to0$ a.s. Thus, for all large $n$, both $(A(c),F(c))$ and $(\widehat A_n(c),\widehat F_n(c))$ belong to the compact set $K_\alpha=[-B_A-1,B_A+1]\times[\alpha/2,1-\alpha/2]$.
Since $q$ is continuously differentiable on $K_\alpha$, it is Lipschitz there. Therefore, for some $L_\alpha<\infty$,
\begin{align*}
\begin{aligned}
\sup_{\alpha\le \widehat F_n(c)\le1-\alpha}
\widehat G_n(c)
&\le
\sup_c G(c)
+
L_\alpha
\left\{
\sup_c|\widehat A_n(c)-A(c)|
+
\sup_c|\widehat F_n(c)-F(c)|
\right\} \\
&=
M_*(X)+o(1)
\qquad\text{a.s.}
\end{aligned}
\end{align*}
where $o(1)$ denotes a term converging to zero almost surely. Combining the left boundary, right boundary, and interior bounds, we get
\begin{align*}
\limsup_n M_n(X)
\le
\max\left\{
\frac{\varepsilon}{2},
M_*(X),
\frac{\varepsilon}{2}
\right\}
\le
M_*(X)+\varepsilon
\qquad\text{a.s.}
\end{align*}
Since $\varepsilon>0$ is arbitrary,
\begin{align*}
\limsup_n M_n(X)\le M_*(X)
\qquad\text{a.s.}
\end{align*}
Together with the lower bound, $\liminf_n M_n(X)\ge M_*(X)$, we conclude that
\begin{align*}
&M_n(X)\to M_*(X)
\qquad\text{a.s.} \qedhere
\end{align*}
\end{proof}


\subsection{Proof of \cref{thrm:nonnull_allocation}}

\begin{proof}
Write $v_Z=\operatorname{Var}(g(Z))$, $G_X=M_*(X)$, and $\varepsilon_0=(v_Z-G_X)/3>0$. Let 
\begin{align*}
R_n=\{\text{every tree in }\mathcal T\text{ splits its root node on }Z\}.
\end{align*} 
We first prove that $\mathbb P(R_n)\to1$. Define the event $B_n=\{|M_n(Z)-v_Z|<\varepsilon_0\}\cap\{|M_n(X)-G_X|<\varepsilon_0\}$. On $B_n$, using $v_Z-\varepsilon_0=G_X+2\varepsilon_0$, we have
\begin{align*}
M_n(Z)
>
v_Z-\varepsilon_0
=
G_X+2\varepsilon_0
>
G_X+\varepsilon_0
>
M_n(X).
\end{align*}
Thus, on $B_n$, the empirical gain of the split on $Z$ is larger than the best empirical threshold gain of $X$. Every tree in $\mathcal T$ is fit to the same sample and splits its root node on the variable with the largest empirical gain, so on $B_n$ every tree splits its root on $Z$; that is, $B_n\subseteq R_n$. Therefore,
\begin{align*}
\mathbb P(R_n^c)
\le
\mathbb P(B_n^c)
\le
\mathbb P\bigl(|M_n(Z)-v_Z|\ge\varepsilon_0\bigr)
+
\mathbb P\bigl(|M_n(X)-G_X|\ge\varepsilon_0\bigr).
\end{align*}
By \cref{lem:cat_gain} and \cref{lem:cont_gain}, $M_n(Z)\to v_Z$ a.s. and $M_n(X)\to G_X$ a.s. Hence, $\mathbb P(R_n^c)\to0$.
We next prove the convergence of $\Phi_n(Z)$. For a tree $T\in\mathcal T$, write $\widehat\Phi_n^{\,T}(X)$ and $\widehat\Phi_n^{\,T}(Z)$ for its contributions to \eqref{eq:mdi}, so that $\Phi_n(\cdot)=|\mathcal T|^{-1}\sum_{T\in\mathcal T}\widehat\Phi_n^{\,T}(\cdot)$. Fix $T\in\mathcal T$. On the event $R_n$, the root split of $T$ separates the two categories $z_0$ and $z_1$, so $Z$ is constant within each child node and cannot be used again at any descendant node. The only term of $\widehat\Phi_n^{\,T}(Z)$ is therefore the root term, which equals $M_n(Z)$, regardless of the remaining structure of $T$. Averaging over $\mathcal T$,
$\Phi_n(Z)=M_n(Z)$ on $R_n$.
For any $\varepsilon>0$,
\begin{align*}
\mathbb P\bigl(
|\Phi_n(Z)-v_Z|>\varepsilon
\bigr)
\le
\mathbb P\bigl(
|M_n(Z)-v_Z|>\varepsilon
\bigr)
+
\mathbb P(R_n^c).
\end{align*}
The first term converges to zero by \cref{lem:cat_gain}, and the second converges to zero by the root-split result above. Therefore,
\begin{align*}
\Phi_n(Z)\overset{p}{\longrightarrow}v_Z
=
\operatorname{Var}(g(Z)).
\end{align*}
Now consider $\Phi_n(X)$. Each tree $T\in\mathcal T$ is fully grown, so every leaf has zero impurity; since $X$ and $Z$ are its only candidate splitting variables, \cref{lem:total_impurity_decrease} applied to $T$ gives $\widehat\Phi_n^{\,T}(X)+\widehat\Phi_n^{\,T}(Z)=\mathcal I([n])$. Averaging over $\mathcal T$, $\Phi_n(X)+\Phi_n(Z)=\mathcal I([n])$. By the definition of variance impurity at the root, $\mathcal I([n])=n^{-1}\sum_{i=1}^n y_i^2-\bar Y_n^2$. Since $\mathbb E[Y^2]<\infty$, the strong law of large numbers gives $\mathcal I([n])\xrightarrow{a.s.}\operatorname{Var}(Y)$.
By the mutual independence of $X$, $Z$, and $\varepsilon$,
\begin{align*}
\operatorname{Var}(Y)
=
\operatorname{Var}(f(X))
+
\operatorname{Var}(g(Z))
+
\sigma^2.
\end{align*}
Therefore,
\begin{align*}
\Phi_n(X)
=
\mathcal I([n])-\Phi_n(Z)
\overset{p}{\longrightarrow}
\operatorname{Var}(Y)-v_Z
=
\operatorname{Var}(f(X))+\sigma^2.
\end{align*}
Finally,
\begin{align*}
\Phi_n(X)+\Phi_n(Z)
=
\mathcal I([n])
\overset{p}{\longrightarrow}
\operatorname{Var}(Y)
=
\operatorname{Var}(f(X))
+
\operatorname{Var}(g(Z))
+
\sigma^2.
\end{align*}
The continuous mapping theorem gives
\begin{align*}
\frac{\Phi_n(Z)}
{\Phi_n(X)+\Phi_n(Z)}
\overset{p}{\longrightarrow}
\frac{\operatorname{Var}(g(Z))}
{\operatorname{Var}(f(X))+\operatorname{Var}(g(Z))+\sigma^2},
\end{align*}
and the normalized importance of $X$ converges to its complement.
\end{proof}


\section{Additional simulation details and results}\label{supsec:additional_simulations}



\subsection{Data-generating process}\label{supsec:sim_dgp}


Each trial draws $n=200$ observations of $p=25$ continuous predictors $X_1,\ldots,X_{25}$ and $q=25$ binary predictors $Z_1,\ldots,Z_{25}$, with $Z_k\overset{\mathrm{iid}}{\sim}\operatorname{Bernoulli}(1/2)$. In the independent design, $X_j\overset{\mathrm{iid}}{\sim}N(0,1)$, independent of the binary predictors. In the correlated design, the $j$th continuous predictor is generated from the $j$th binary predictor as
\begin{align*}
X_j = (2Z_j-1) + \eta_j, \qquad \eta_j\overset{\mathrm{iid}}{\sim}N(0,1),
\end{align*}
so that $X_j\mid Z_j=1\sim N(1,1)$ and $X_j\mid Z_j=0\sim N(-1,1)$, giving $\operatorname{Corr}(X_j,Z_j)=1/\sqrt{2}$. The $25$ pairs are mutually independent, and continuous predictors carry larger marginal variance in the correlated design than in the independent one.

In each trial, five continuous and five binary predictors are sampled uniformly without replacement to form the active sets $S_X$ and $S_Z$, independently of the correlated-design pairing. Each active predictor receives an independent sign $r\in\{-1,+1\}$, uniform, drawn once per trial. The response is
\begin{align*}
Y = \sum_{j\in S_X} r_j\,\psi(X_j) + 2\sum_{k\in S_Z} r_k\,Z_k + \varepsilon, \qquad \varepsilon\sim N(0,\sigma^2),
\end{align*}
with link $\psi(x)=x$ in the linear model and
\begin{align}\label{eq:exponential}
\psi(x) = \frac{\exp(-x^2/2) - 1/\sqrt{2}}{\sqrt{1/\sqrt{3} - 1/2}}
\end{align}
in the nonlinear model. The constants $1/\sqrt{2}$ and $\sqrt{1/\sqrt{3}-1/2}$ are the mean and standard deviation of $\exp(-X^2/2)$ for $X\sim N(0,1)$, so $\psi(X)$ has mean zero and unit variance under the independent design; since $\operatorname{Var}(2Z_k)=1$, all terms are then on a common scale. Within each trial, $\sigma^2$ is set to the empirical variance of the signal term across the $n$ observations, giving a signal-to-noise ratio of one.


\subsection{Random forest ROC curves}\label{supsec:additional_ranking}


\cref{fig:roc_rf} shows the average ROC curves from which the AUC values in \cref{tab:study1_auc} are derived. XGBoost and jittered XGBoost, which performed worse than several of the forest-based approaches, are omitted to make the figure easier to read.

\begin{figure}[H]
    \centering
    \includegraphics[width=\linewidth]{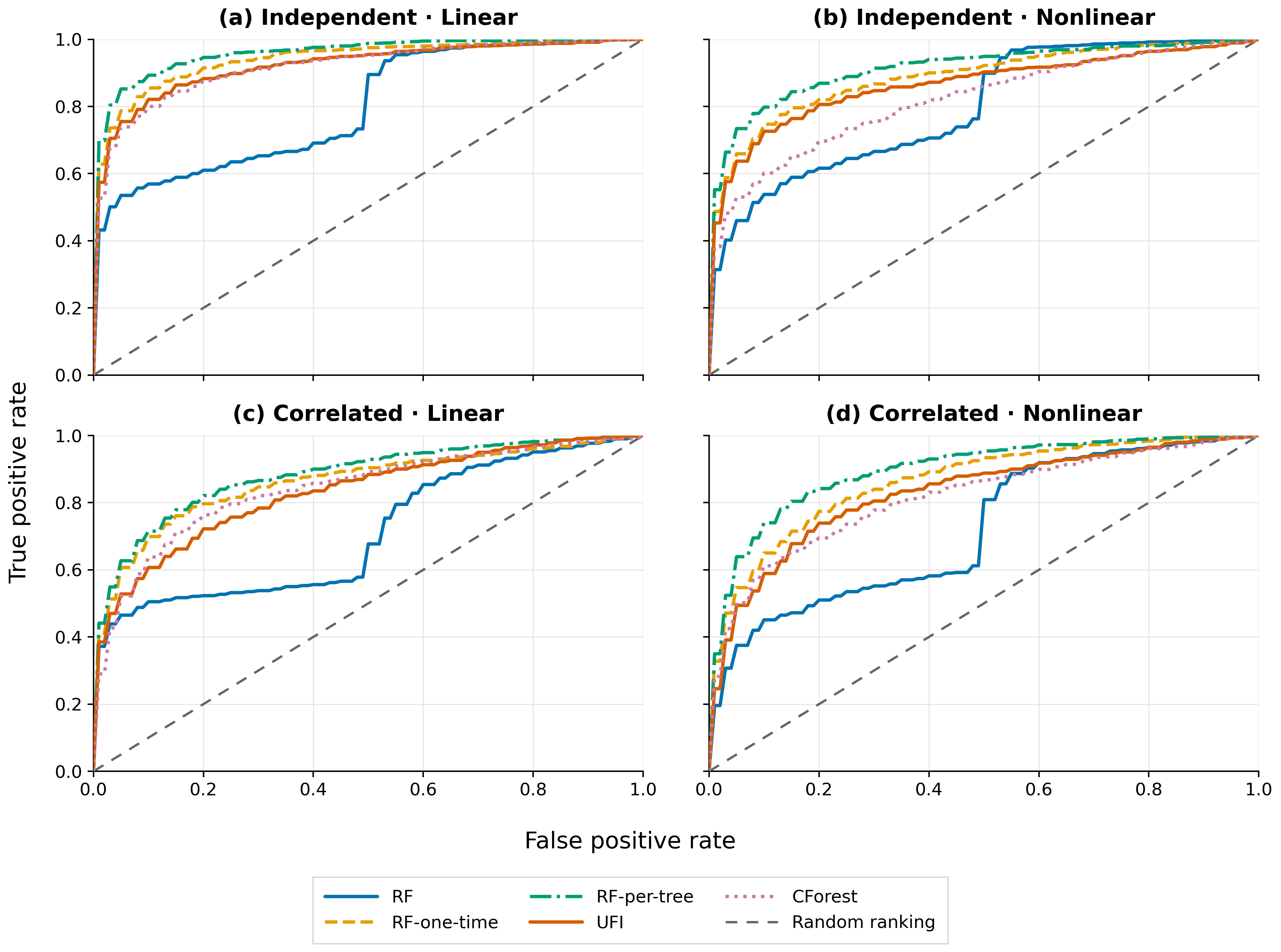}
    \caption{
    \textit{ROC curves}. Panels correspond to the independent-linear,
    independent-nonlinear, correlated-linear, and correlated-nonlinear
    simulation designs. Curves closer to the upper-left corner indicate better
    ranking of active predictors above inactive predictors; the dashed
    diagonal represents random ranking.
    }
    \label{fig:roc_rf}
\end{figure}

Standard RF's curve has a distinctive shape: it rises only modestly until the false positive rate reaches about $0.5$, then jumps sharply. This reflects the ranking behavior noted in the main text: RF places all $25$ continuous predictors, true and false alike, above every binary predictor, so the five active binary predictors are not identified until roughly ranks $26$--$30$.


\subsection{Time analysis}
\label{supsec:time_analysis}


We compare computation times of RF, RF-one-time,
RF-per-tree, UFI, and CForest under the independent-covariate, linear-signal setting of \cref{sec:simulation_design}. For fair comparison, each procedure uses 100 trees, all $p$ predictors are eligible at each split, the minimum number of observations required to split a node is two, and the minimum terminal-node size is one. Each tree is trained using a bootstrap sample of $n$ observations drawn with replacement. Reported values are the mean wall-clock time in seconds over five trials. CForest is run only up to $n=1000$ because its runtime made the larger sample sizes impractical. Results are reported in \cref{tab:runtime_n}.

\begin{table}[htbp]
\centering
\setlength{\tabcolsep}{14pt}
\begin{tabular}{cccccc}
\toprule
$n$ & RF & RF-one-time & RF-per-tree & UFI & CForest \\
\midrule
250 & 0.3 & 0.4 & 0.4 & 0.9 & 88.0 \\
500 & 0.6 & 0.9 & 0.9 & 2.4 & 181.9 \\
1000 & 1.2 & 1.9 & 2.0 & 7.0 & 362.8 \\
2000 & 1.8 & 2.8 & 2.9 & 13.8 & -- \\
4000 & 3.9 & 6.3 & 6.5 & 49.3 & -- \\
8000 & 9.0 & 14.8 & 14.9 & 184.2 & -- \\
\bottomrule
\end{tabular}
\caption{\textit{Runtime as a function of sample size}. Mean runtime in seconds over five trials, with $p=50$ and matched tree-specific parameters; trial-to-trial variation was negligible. CForest was run only up to $n=1000$; dashes indicate settings where its runtime made the comparison impractical.}
\label{tab:runtime_n}
\end{table}


\section{Bladder cancer analysis}\label{supsec:blca}


We use the bladder urothelial carcinoma cohort of The Cancer Genome Atlas \citep{weinstein2013cancer}, accessed through LinkedOmics \citep{vasaikar2018linkedomics}. Clinical records are available for $412$ patients and tumor microRNA profiles for $409$; merging on patient identifier gives $409$ patients with both (\cref{tab:blca_cohort}). The outcome is survival at day $548$ ($18$ months), derived from overall survival time and vital status: a patient known to be alive on or after day $548$ is a survivor, one recorded dead on or before it is not, and one last known alive earlier with no recorded death has undetermined status and is excluded, so censoring before day $548$ is handled by exclusion rather than modeled. We drop the $100$ patients with undetermined status, as well as those with unrecorded stage. Just two patients in the entire cohort are Stage~I, too few to treat as a separate category, which is why our analysis covers Stages~II, III, and IV only. The final cohort has $n=306$ patients: $192$ survivors and $114$ deaths within $18$ months; $88$, $103$, and $115$ patients are Stage~II, III, and IV, respectively. For predictors, we start from $426$ microRNA features, keep the $425$ that are complete across all $409$ merged patients, and greedily remove one from each pair with absolute Pearson correlation above $0.999$, leaving $423$. No transformation or standardization is applied. Together with age and pathologic stage, the dataset comprises $425$ predictors. Pathologic stage is the only categorical predictor.

\begin{table}[htbp]
\centering
\setlength{\tabcolsep}{12pt}
\begin{tabular}{lcc}
\toprule
Step & Removed & Remaining \\
\midrule
Patients with clinical and microRNA data & --- & 409 \\
Undetermined 18-month survival status & 100 & 309 \\
Stage~I or unrecorded stage & 3 & 306 \\
\bottomrule
\end{tabular}
\caption{\textit{Cohort construction}. Complete-case exclusions applied in order to the $409$ patients with both clinical and microRNA data.}
\label{tab:blca_cohort}
\end{table}


\section{Machine learning datasets}\label{supsec:ml}


The seven datasets are all binary classification tasks; their sizes and predictor counts are in \cref{tab:ipss_null_spikein}. Cylinder Banding, German Credit, and QSAR Biodegradation are from OpenML \citep{vanschoren2014openml}, and SAheart, Titanic, Hepatitis, and Hypothyroid from the Penn Machine Learning Benchmark \citep{Olson2017PMLB}. Cylinder Banding concerns banding defects in rotogravure printing; SAheart and Hypothyroid predict coronary heart disease and thyroid status from clinical measurements, respectively; Titanic predicts passenger survival during the sinking of the Titanic; Hepatitis predicts patient survival; German Credit predicts credit risk; and QSAR Biodegradation predicts whether a chemical is readily biodegradable from molecular descriptors.

Predictor types are set based on observed values: a numeric predictor with more than $20$ distinct values is treated as continuous; those with 20 or fewer distinct values or with non-numeric predictors are categorical. Predictors with only one distinct value are dropped. For Cylinder Banding we also remove date and identifier. Missing continuous values are imputed with the full-data median; missing categorical values form their own category. Categorical predictors are integer-encoded before fitting.

For each dataset we run IPSS using the Python package \texttt{ipss}. Jittered selectors perturb categorical predictors once prior to running IPSS and are otherwise identical to their unperturbed counterparts. \cref{fig:ipss_selection_cardinality} shows how many predictors each selector picks and how many distinct values those predictors have. RF's selections skew toward many-category and continuous predictors relative to selections made by the other four methods.

\begin{figure}[htbp]
\centering
\includegraphics[width=\linewidth]{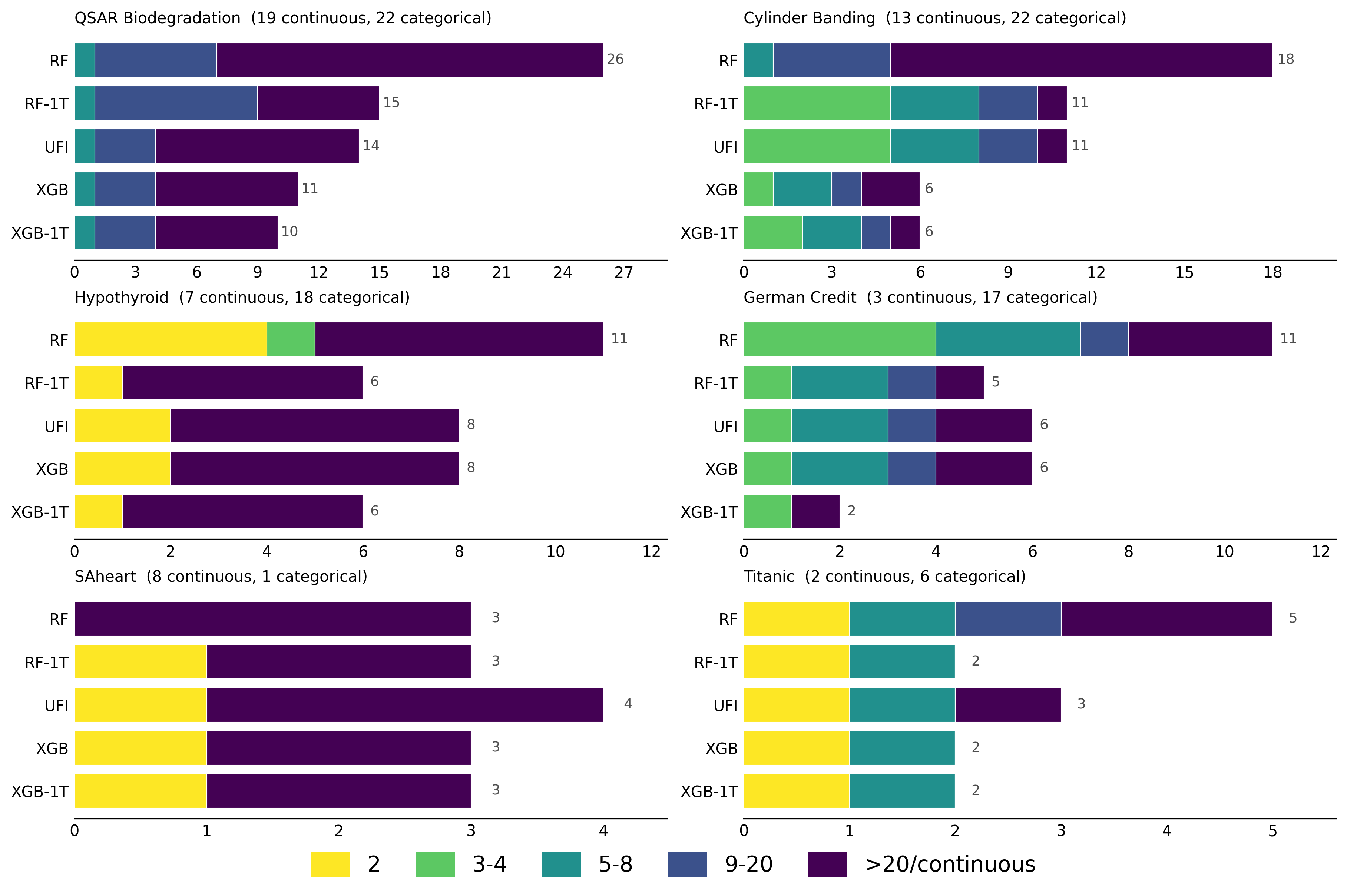}
\caption{\textit{Selections by predictor type}. One panel per dataset (continuous and categorical predictor counts in the title); within each, one horizontal bar per base selector (RF-1T and XGB-1T are the one-time jittered selectors) whose length is the number of predictors selected by IPSS at a nominal FDR level of $0.1$, split by the predictor's number of distinct values, with continuous predictors counted in the top bin ($>20$). RF selects the most predictors on every dataset except SAheart and is dominated by continuous and many-category predictors.}
\label{fig:ipss_selection_cardinality}
\end{figure}

Jittered RF and UFI select similar predictors, as do XGBoost and jittered XGBoost, and all four differ from RF, which makes more selections than the other methods in general. On Cylinder Banding, jittered RF and UFI select all five process controls, each of which has only a few categories; RF selects none of them, instead selecting high-category predictors, such as \texttt{customer}, which takes $71$ values and is essentially a record identifier. On SAheart, every method except RF selects family history (\texttt{Famhist}), the only categorical predictor in this dataset and a known risk factor for coronary heart disease. German Credit and Hypothyroid have only three and seven continuous predictors, so RF selects the categorical predictors even without jittering; here jittering mainly drops the weak predictors and keeps the strong ones, such as \texttt{checking\_status} and \texttt{TSH\_measured}. On Hypothyroid, though, jittering and UFI also drop \texttt{on\_thyroxine} and \texttt{query\_hypothyroid}, both of which are clinically informative. On QSAR Biodegradation, jittering adds atom-count predictors such as \texttt{nO}, \texttt{nN}, and the halogen count \texttt{nX}, which RF skips in favor of descriptors that take many distinct values.

\begin{figure}[htbp]
\centering
\includegraphics[width=\linewidth]{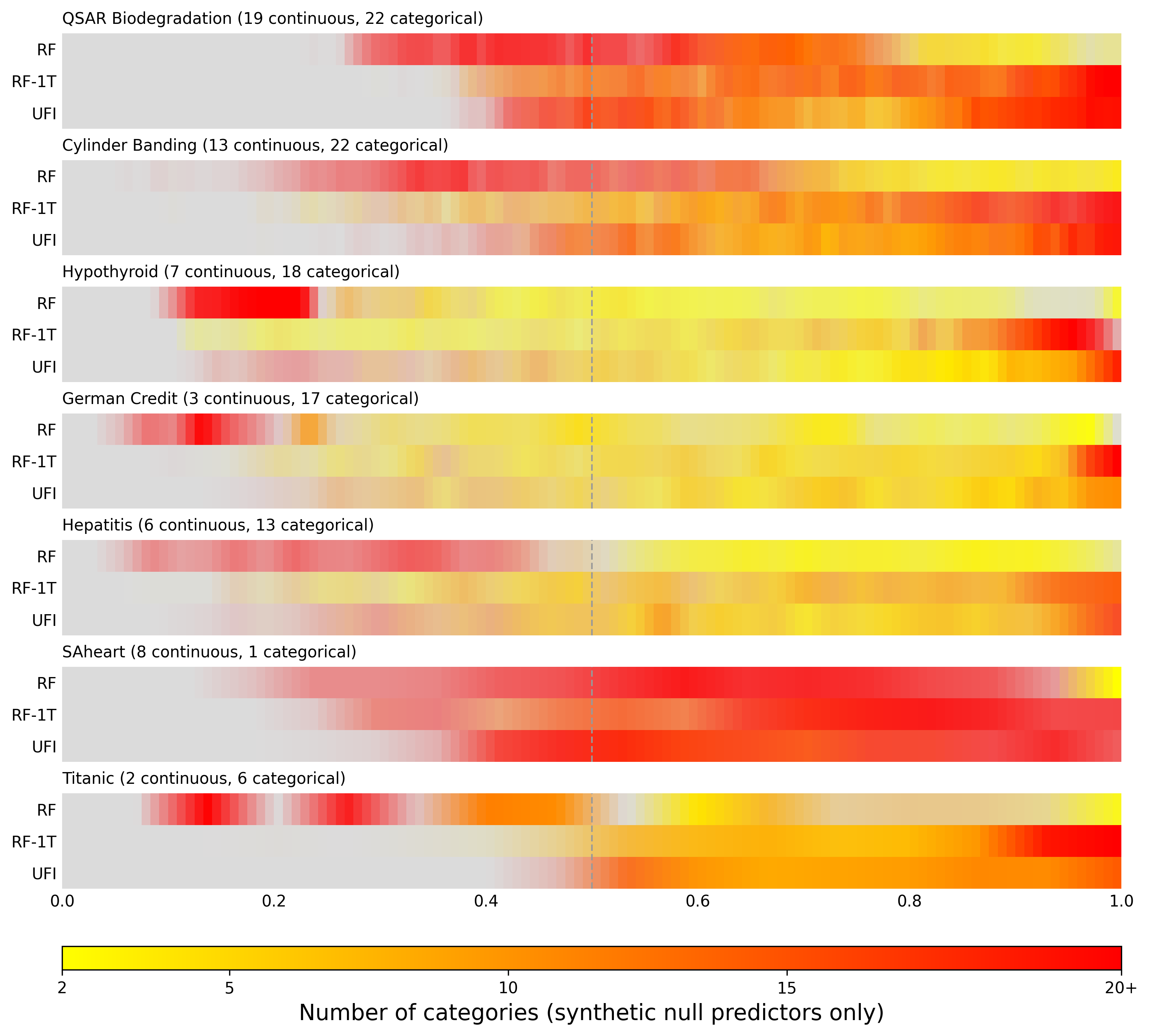}
\caption{\textit{Ranking distributions of synthetic nulls}. One panel per dataset; within each, one row per importance method and one column per rank position (rank $1$ is most important). Observed predictors are grey; synthetic null predictors are coloured by their number of distinct values, from yellow (binary) to red ($20$ or more, and all continuous nulls), averaged over $50$ seeds. Each dataset is augmented with one permuted null per observed predictor, so half of every row is null; the dashed line marks the midpoint. Under standard RF the nulls reach into the top half in order of cardinality; one-time jittering and UFI push them below the midpoint.}
\label{fig:rank_null_spikein}
\end{figure}


\section{Sensitivity analyses}\label{supsec:sensitivity}


We test whether the jittered methods are sensitive to the jitter strength $\delta$. We rerun the four settings of the main AUC study (\cref{sec:simulation_auc})---independent or correlated covariates crossed with linear or nonlinear signals---for RF-one-time and XGB-one-time at $\delta\in\{10^{-6},10^{-4},10^{-2},10^{-1}\}$, with $100$ trials per setting and every $\delta$ evaluated on the same data, response, and active set within a trial.

\begin{table*}[htbp]
\centering
\resizebox{\textwidth}{!}{
\begin{tabular}{llcccc}
\toprule
Dependence and signal
& Method
& $\delta=10^{-6}$
& $\delta=10^{-4}$
& $\delta=10^{-2}$
& $\delta=10^{-1}$ \\
\midrule
Independent--linear
& RF-one-time
& 0.945 (0.039)
& 0.944 (0.041)
& 0.944 (0.040)
& 0.944 (0.040) \\
& XGB-one-time
& 0.843 (0.066)
& 0.842 (0.062)
& 0.838 (0.064)
& 0.838 (0.064) \\
\addlinespace

Independent--nonlinear
& RF-one-time
& 0.891 (0.059)
& 0.891 (0.058)
& 0.890 (0.058)
& 0.890 (0.058) \\
& XGB-one-time
& 0.789 (0.076)
& 0.777 (0.076)
& 0.774 (0.074)
& 0.774 (0.074) \\
\addlinespace

Correlated--linear
& RF-one-time
& 0.852 (0.067)
& 0.857 (0.067)
& 0.858 (0.067)
& 0.858 (0.067) \\
& XGB-one-time
& 0.784 (0.085)
& 0.775 (0.081)
& 0.773 (0.081)
& 0.773 (0.081) \\
\addlinespace

Correlated--nonlinear
& RF-one-time
& 0.862 (0.062)
& 0.862 (0.063)
& 0.862 (0.064)
& 0.861 (0.063) \\
& XGB-one-time
& 0.765 (0.075)
& 0.767 (0.078)
& 0.768 (0.079)
& 0.768 (0.079) \\
\bottomrule
\end{tabular}
}
\caption{\textit{Feature-recovery AUC under different jitter strengths}. Entries are the mean (standard deviation) across $100$ trials.}
\label{tab:jitter_sensitivity}
\end{table*}

\begin{figure}[htbp]
    \centering
    \includegraphics[width=\linewidth]{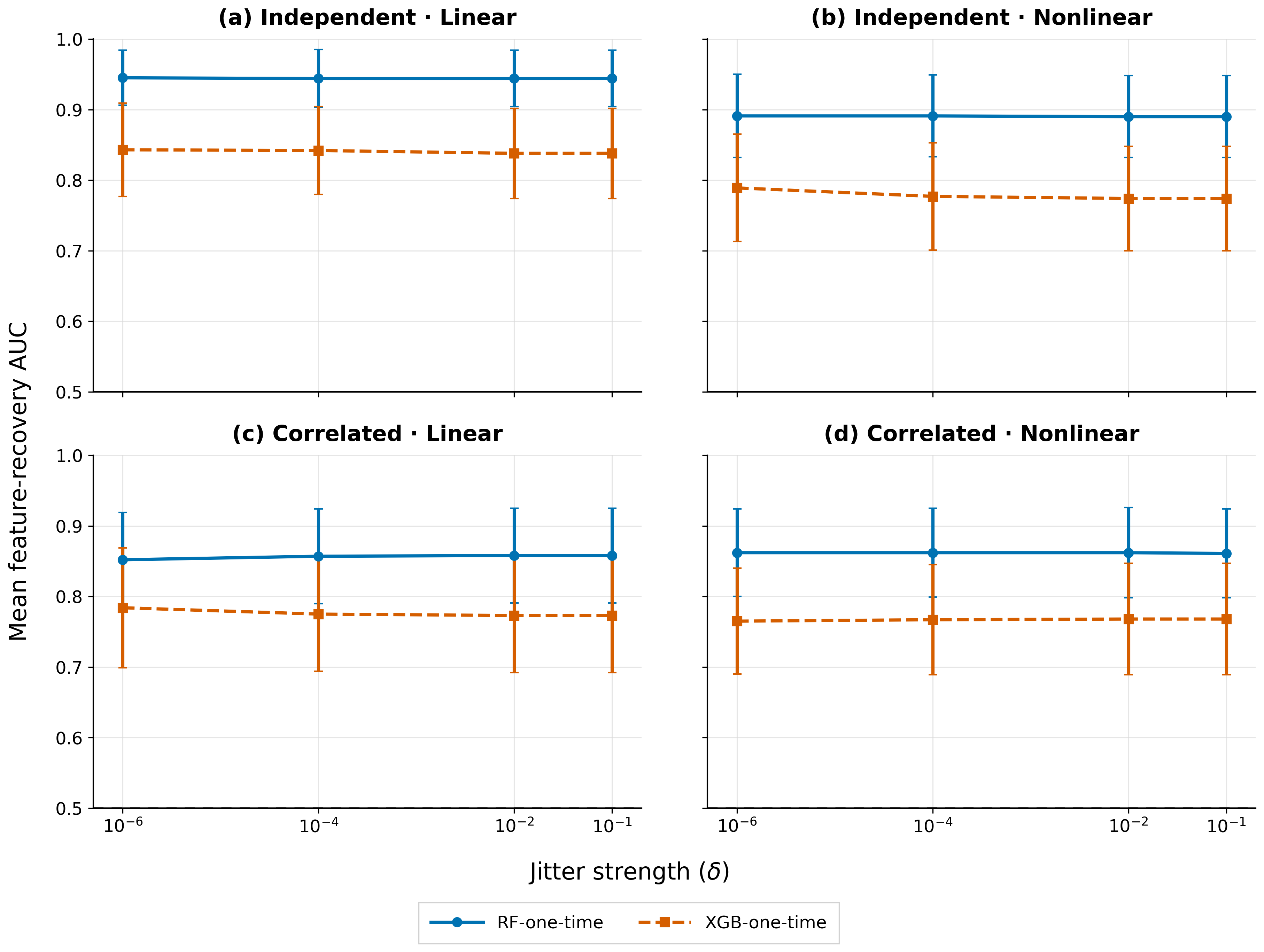}
    \caption{\textit{Sensitivity of selection to jitter strength.} Average AUC in each of the four simulation settings across different jitter strengths, $\delta$. Points are the mean AUC over $100$ trials; error bars are one standard deviation.}
\label{fig:jitter_sensitivity}
\end{figure}

Feature-recovery AUC is essentially flat in $\delta$ (\cref{tab:jitter_sensitivity}, \cref{fig:jitter_sensitivity}). Within each setting, the gap between the largest and smallest mean AUC is at most $0.006$ for RF-one-time and $0.016$ for XGB-one-time, far smaller than the trial-to-trial variation and than the gap between the two methods. The default $\delta=10^{-4}$ performs like every other value tested, so it does not require tuning.


\section{Prediction performance}\label{supsec:prediction}


To check whether jittering harms predictive accuracy, we compare one-time-jittered and standard tree ensembles on 51 datasets from the Penn Machine Learning Benchmark \citep{Olson2017PMLB}. Each dataset has at least one continuous and one categorical predictor, fewer than $15{,}000$ samples, and at least five predictors. The collection includes both classification and regression tasks. 

For each dataset we fit a random forest and an XGBoost model and evaluate the test error---misclassification rate for classification, mean squared error for regression---averaged over five cross-validation folds, which are stratified by the response for classification tasks. We then repeat the evaluation with a one-time uniform jitter of strength $\delta=10^{-4}$ and record the error ratio, defined as the jittered error divided by the unjittered error. A ratio below one means jittering improves predictive accuracy; a ratio above one means it degrades it.

\cref{fig:prediction_error_pmlb} shows the distribution of these error ratios across datasets. After removing datasets on which both models are near-perfectly accurate (mean error below $0.05$ for both unjittered and jittered models), $33$ datasets remain ($22$ classification, $11$ regression), ranging from $47$ to $9{,}822$ samples and from $5$ to $85$ predictors. For random forests, the median error ratio is $1.00$ with IQR $[0.93, 1.02]$; for XGBoost it is $0.99$ with IQR $[0.94, 1.03]$. Adding noise to categorical predictors thus has no systematic effect on predictive accuracy, consistent with the stability of the importance rankings in \cref{supsec:sensitivity} and with the claim in \cref{sec:jitter_prediction}.

\begin{figure}[htbp]
\centering
\includegraphics[width=0.85\textwidth]{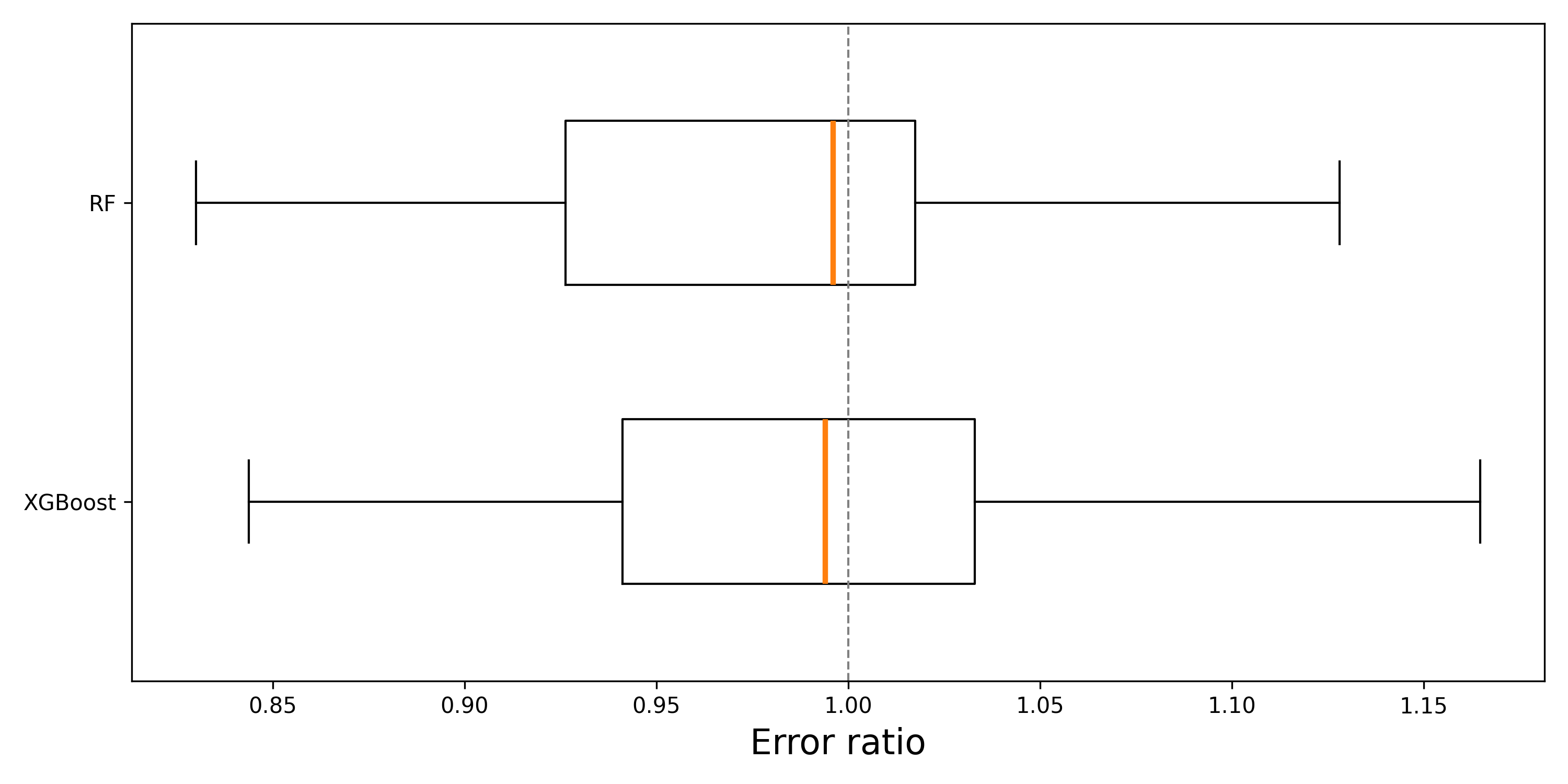}
\caption{\textit{Jittering has little effect on predictive accuracy}. Distribution over Penn Machine Learning Benchmark (PMLB) datasets of the error ratio between jittered and unjittered ensembles, for random forests (top) and XGBoost (bottom). The dashed gray line at an error ratio of one marks no change; values below one indicate that jittering reduces error, while values above indicate an increase.}
\label{fig:prediction_error_pmlb}
\end{figure}

\FloatBarrier

\putbib[reference]
\end{bibunit}

\end{document}